%% file: main.tex
\documentclass[11pt]{article}

\usepackage[numbers,sort]{natbib}
\usepackage{bm}
\usepackage{fullpage}
\usepackage[english]{babel}
\usepackage[utf8x]{inputenc}
\usepackage[T1]{fontenc}
\usepackage{enumitem}
\setlist{nosep}
\usepackage{microtype}

\usepackage{etoolbox}
\apptocmd{\sloppy}{\hbadness 10000\relax}{}{}

\usepackage{amsmath,amsthm,amssymb,mathtools,thmtools}
\usepackage{graphicx}
\usepackage[textsize=scriptsize,disable]{todonotes}
\usepackage[colorlinks=true, allcolors=blue]{hyperref}
\usepackage{algorithm,algorithmicx,algpseudocode}
\usepackage{xfrac}
\usepackage[normalem]{ulem}

\newcommand{\hassan}[1]{\todo[inline,color=green!30]{\small\textbf{Hassan:} #1}}
\newcommand{\chris}[1]{\todo[inline,color=red!30]{\small\textbf{Chris:} #1}}
\newcommand{\chere}[1]{\chris{I am here.}}

\newcommand{\modify}[1]{{#1}}

\newcommand{\wtilde}{\widetilde}

\newcommand\blfootnote[1]{%
  \begingroup
  \renewcommand\thefootnote{}\footnote{#1}%
  \addtocounter{footnote}{-1}%
  \endgroup
}

\input{notation}

\title{Private and polynomial time algorithms for learning Gaussians and beyond} 

\author{
    Hassan Ashtiani\thanks{McMaster University, \texttt{zokaeiam@mcmaster.ca}. Hassan Ashtiani is also affiliated with Vector institute and was supported by an NSERC Discovery Grant.}
    \and
    Christopher Liaw\thanks{Google, \texttt{cvliaw@google.com}. Part of this work was done while at the University of Toronto.}
}

\numberwithin{equation}{section}

\begin{document}

\maketitle
\begin{abstract}
\input{abstract.tex}
\end{abstract}

\blfootnote{Accepted for presentation at the Conference on Learning Theory (COLT) 2022 \cite{AshtianiL22}.}

\input{intro.tex}
\input{prelim.tex}
\input{reduction.tex}
\input{subspace.tex}
\input{covariance.tex}
\input{robust_cov.tex}
\input{mean_robust.tex}

\bibliographystyle{plain}
\bibliography{refs}

\appendix
\input{app_facts.tex}
\input{app_prelim.tex}
\input{app_reduction.tex}
\input{app_covariance.tex}

\end{document}

%% file: notation.tex
\newcommand{\bE}{\mathbb{E}}

\newcommand{\bN}{\mathbb{N}}

\newcommand{\bP}{\mathbb{P}}
\newcommand{\bR}{\mathbb{R}}
\newcommand{\bS}{\mathbb{S}}

\newcommand{\cA}{\mathcal{A}}
\newcommand{\cB}{\mathcal{B}}
\newcommand{\cD}{\mathcal{D}}

\newcommand{\cL}{\mathcal{L}}
\newcommand{\cM}{\mathcal{M}}
\newcommand{\cN}{\mathcal{N}}

\newcommand{\cS}{\mathcal{S}}

\newcommand{\cX}{\mathcal{X}}
\newcommand{\cY}{\mathcal{Y}}

\newcommand{\eps}{\varepsilon}

\newcommand{\KL}[2]{D_{\mathrm{KL}}\left( {#1} \parallel {#2} \right)}
\newcommand{\privLoss}[2]{\cL_{{#1} \parallel {#2}}}

\newcommand{\dtv}[2]{d_{\mathrm{TV}}\left( {#1}, {#2} \right)}

\DeclareMathOperator{\vvec}{vec}
\DeclareMathOperator{\bd}{bd}

\newcommand{\eqdist}{\stackrel{d}{=}}
\newcommand{\what}{\widehat}
\newcommand{\wmu}{\widehat{\mu}}
\newcommand{\wSigma}{\widehat{\Sigma}}

\newtheorem{theorem}{Theorem}[section]
\newtheorem{lemma}[theorem]{Lemma}
\newtheorem{claim}[theorem]{Claim}
\newtheorem{remark}[theorem]{Remark}
\newtheorem{proposition}[theorem]{Proposition}
\newtheorem{fact}[theorem]{Fact}
\newtheorem{corollary}[theorem]{Corollary}
\theoremstyle{definition}
\newtheorem{definition}[theorem]{Definition}

\newcommand{\AlgorithmName}[1]{\label{alg:#1}}
\newcommand{\Algorithm}[1]{Algorithm~\ref{alg:#1}}
\newcommand{\AppendixName}[1]{\label{app:#1}}
\newcommand{\Appendix}[1]{Appendix~\ref{app:#1}}
\newcommand{\ClaimName}[1]{\label{claim:#1}}
\newcommand{\Claim}[1]{Claim~\ref{claim:#1}}
\newcommand{\CorollaryName}[1]{\label{corollary:#1}}
\newcommand{\Corollary}[1]{\autoref{corollary:#1}}
\newcommand{\DefinitionName}[1]{\label{def:#1}}
\newcommand{\Definition}[1]{Definition~\ref{def:#1}}
\newcommand{\EquationName}[1]{\label{eq:#1}}
\newcommand{\Equation}[1]{Eq.~\eqref{eq:#1}}
\newcommand{\FactName}[1]{\label{fact:#1}}
\newcommand{\Fact}[1]{\autoref{fact:#1}}
\newcommand{\LemmaName}[1]{\label{lem:#1}}
\newcommand{\Lemma}[1]{\autoref{lem:#1}}
\newcommand{\LineName}[1]{\label{line:#1}}
\newcommand{\Line}[1]{Line~\ref{line:#1}}
\newcommand{\PropositionName}[1]{\label{prop:#1}}
\newcommand{\Proposition}[1]{Proposition~\ref{prop:#1}}
\newcommand{\SectionName}[1]{\label{sec:#1}}
\newcommand{\Section}[1]{Section~\ref{sec:#1}}
\newcommand{\SubsectionName}[1]{\label{sec:#1}}
\newcommand{\Subsection}[1]{Subsection~\ref{sec:#1}}
\newcommand{\TheoremName}[1]{\label{thm:#1}}
\newcommand{\Theorem}[1]{\autoref{thm:#1}}
\newcommand{\pd}{\bS_{> 0}}
\newcommand{\psd}{\bS}
\DeclareMathOperator*{\argmax}{arg\,max}

\DeclareMathOperator{\TLap}{TLap}

\newcommand{\prob}[1]{\bP \left[#1\right]}
\newcommand{\probs}[2]{\bP_{#1} \left[#2\right]}

\newcommand{\expect}[1]{\bE \left[#1\right]}
\newcommand{\expects}[2]{\bE_{#1} \left[#2\right]}

\newcommand{\transpose}{^\top}

\newcommand{\dd}{\mathrm{d}}

\DeclareMathOperator{\dist}{dist}

\DeclareMathOperator{\Span}{span}
\DeclareMathOperator{\poly}{poly}
\DeclareMathOperator{\polylog}{polylog}

\DeclareMathOperator{\Tr}{tr}
\DeclareMathOperator{\rank}{rank}

\makeatletter
\algnewcommand{\LineComment}[1]{\Statex \hskip\ALG@thistlm \(\triangleright\) #1}
\makeatother

\newcommand{\comment}[1]{}

%% file: abstract.tex
We present a fairly general framework for reducing $(\varepsilon, \delta)$ differentially private (DP) statistical estimation to its non-private counterpart.
As the main application of this framework, we give a polynomial time and $(\varepsilon,\delta)$-DP algorithm for learning (unrestricted) Gaussian distributions in $\bR^d$. The sample complexity of our approach for learning the Gaussian up to total variation distance $\alpha$ is
$\wtilde{O}(d^2/\alpha^2 + d^2\sqrt{\ln(1/\delta)}/\alpha \eps + d\ln(1/\delta) / \alpha \eps)$
matching (up to logarithmic factors) the best known information-theoretic (non-efficient) sample complexity upper bound due to Aden-Ali, Ashtiani, and Kamath \cite{aden2021sample}.
In an independent work, Kamath, Mouzakis, Singhal, Steinke, and Ullman~\cite{KMSTU21} proved a similar result using a different approach and with $O(d^{5/2})$ sample complexity dependence on $d$.

As another application of our framework, we provide the first polynomial time $(\varepsilon, \delta)$-DP algorithm for robust learning of (unrestricted) Gaussians with sample complexity $\wtilde{O}(d^{3.5})$.
In another independent work, Kothari, Manurangsi, and Velingker~\cite{KMV21} also provided a polynomial time $(\eps, \delta)$-DP algorithm for robust learning of Gaussians
with sample complexity $\wtilde{O}(d^8)$.

%% file: intro.tex
\section{Introduction}
\chris{Cite Pravesh Kothari paper.}
Learning a multivariate Gaussian distribution with respect to the total variation (TV) distance is one of the most basic tasks in statistical estimation. It is folklore that simply taking the empirical mean and covariance matrix achieves the optimal sample complexity of $\Theta(d^2/\alpha^2)$ for learning \emph{unbounded} Gaussians in $\mathbb{R}^d$ up to TV distance $\alpha$ (see, for example, \cite[Theorem C.1]{ABHLMP20}).
However, this simple approach can compromise the privacy of the individuals/entities whose data has been used in the estimation procedure. Differential privacy~\cite{DMNS06} offers a formal framework to study and guarantee privacy in such estimation tasks. The central question that we aim to address is whether unbounded Gaussians can be learned \emph{privately} and \emph{efficiently}.

For the strong notion of \emph{pure} ($\varepsilon, 0$)-differential privacy, learning unbounded Gaussians is impossible. In fact, even learning the mean of univariate Gaussians with unit variances requires some boundedness assumptions on the parameters of the Gaussian~\cite[Theorem 1.4]{karwa2018finite}.
This motivates the study of learning Gaussians under the \emph{approximate} ($\varepsilon, \delta$)-differential privacy model~\cite{dwork2006our} (or approximate-DP for short).

Previous work has established the polynomial-time learnability of univariate Gaussians (with unbounded mean and variance) in the approximate-DP model. The seminal work of Karwa and Vadhan~\cite{karwa2018finite} uses stability-based histograms~\cite{BNS19} to find a bound on the parameters; once the problem is reduced to the bounded case, the Laplace mechanism~\cite{DMNS06} offers a private way to release a noisy---but sufficiently accurate---estimate of the parameters of the Gaussian\footnote{Although the method of \cite{karwa2018finite} works with a histogram with an infinite number of bins, it can be easily implemented efficiently by only constructing the bins that are non-empty (i.e., lazy initialization).}.

The problem is more difficult in the multivariate setting.
Although one could apply the method of \cite{karwa2018finite} along each axis,
this would only yield an accurate solution for axis-aligned Gaussians,
i.e.~Gaussians whose covariance matrix is a diagonal matrix.

The first result for multivariate Gaussians is due to \cite{kamath2019privately}.
They gave an algorithm for learning multivariate Gaussians in the concentrated DP model\footnote{
    Concentrated DP is a weaker notion of privacy than pure DP but stronger than approx DP \citep{dwork2016concentrated, bun2016concentrated}.
}.
However, their algorithm requires prior knowledge of the condition number of the covariance matrix and its sample complexity depends polylogarithmically
on the condition number.
In a follow-up work, \cite{biswas2020coinpress} proposed the CoinPress algorithm which is a simple and practical algorithm
for learning multivariate Gaussians in the concentrated DP model with more or less the same sample complexity.

More recently, Aden-Ali, Ashtiani, and Kamath~\cite{aden2021sample} established a sharp upper bound on the sample complexity of learning Gaussians in the approximate-DP model, demonstrating that the dependence on the condition number (or the size of the parameters) can be totally avoided.
Unfortunately, the approach of~\cite{aden2021sample} is non-constructive and does not yield a computationally efficient method. In fact, it only shows the \emph{existence} of a ``locally small'' cover for Gaussians in TV distance, which yields an upper bound on the sample complexity through the use of the private hypothesis selection method~\cite{bun2019private}. This raises the following concrete question.

\begin{quote}
    Is there an $(\varepsilon, \delta)$-DP and polynomial time\footnote{
        By polynomial-time, we mean $\poly(d, 1/\eps, 1/\alpha, \log(1/\delta))$ where $\alpha$ is the accuracy parameter in TV distance.
    }
    algorithm for learning unbounded Gaussian distributions over $\mathbb{R}^d$?
\end{quote}

More ambitiously, we can ask if there is an efficient algorithm for learning high-dimensional Gaussians \emph{robustly},
where a small fraction of the samples are arbitrarily corrupted by an adversary.
The breakthrough of~\cite{diakonikolas2016robust, lai2016agnostic} answered this question positively.
In particular, the work of~\cite{diakonikolas2016robust} offers a polynomial time algorithm for robustly learning Gaussians with respect to the TV distance.
This raises the following natural question.

\begin{quote}
    Is there an $(\varepsilon, \delta)$-DP, \emph{robust}, and polynomial time algorithm for learning unbounded Gaussian distributions over $\mathbb{R}^d$?
\end{quote}

\subsection{Main Results}
In this paper, we devise a simple and general algorithmic technique for private statistical estimation.
Our first application is an efficient algorithm for learning a Gaussian distribution in $\bR^d$.
\begin{theorem}[Informal]
\TheoremName{covariance_informal}
There is an $(\varepsilon,\delta)$-DP algorithm that 
takes $m=\wtilde{O}\left(\frac{d^2}{\alpha^2}+\frac{d^2 \sqrt{\ln{1/\delta}}}{\alpha\varepsilon} + \frac{d\ln(1/\delta)}{\alpha \eps} \right)$ i.i.d. samples from an arbitrary Gaussian distributions over $\mathbb{R}^d$, runs in time polynomial in $m$,
and with high probability outputs a Gaussian whose TV distance to the true distribution is smaller than $\alpha$.
\end{theorem}

The formal version of the theorem is stated in \Theorem{full_gaussian}.
The sample complexity of this approach matches the best (non-efficient) learning method~\cite{aden2021sample} up to logarithmic factors. 
Moreover, \cite[Theorem 56]{kamath2019privately} and \cite[Theorem 1.4]{karwa2018finite} show sample complexity lower bounds of $\Omega(d/\alpha\varepsilon)$ and $\Omega(\log(1/\delta)/\varepsilon)$ for private learning of spherical and univariate Gaussians respectively. Finally, a lower bound of $\Tilde{\Omega}(d^2/\alpha^2)$ is known for the non-private setting~\cite[Theorem 1.2]{ABHLMP20}.

The main ingredient of proving \Theorem{covariance_informal} is a polynomial-time algorithm to compute a preconditioner which brings the condition number
of the covariance matrix down to a constant.
Once we have a preconditioner, we can then (mostly) use known results to prove \Theorem{covariance_informal}.
The algorithm to compute a preconditioner is extremely simple: it computes several candidate covariance matrices from disjoint samples,
computes a weighted average of the ``good'' candidates, and then adds noise to resulting covariance matrix.
Here, a candidate covariance matrix is ``good'' if it is close to a large majority of the other candidate covariance matrices;
intuitively, this should also indicate that it is close to the true covariance matrix.
The first two steps come directly from the framework discussed in \Section{reduction} and the last step (adding the noise) is discussed in \Section{covariance}.

In order to handle Gaussians with degenerate covariance matrices, we need a private and efficient method for learning (exactly) the subspace of a Gaussian.
The recent work of Singhal and Steinke~\cite{SS21} offers a differentially private but inefficient method for this problem.
We develop an efficient method for this problem which may be of independent interest (see~\Theorem{subspace} for a formal statement).
We note that our result is somewhat incomparable to that of Singhal and Steinke~\cite{SS21} since the assumptions are different.

\begin{theorem}[Informal]
\TheoremName{subspace_informal}
There is an $(\eps, \delta)$-DP algorithm that takes $m = O(d \ln(1/\delta) / \eps)$ i.i.d.~samples from a Gaussian distribution in $\bR^d$
and with probability 1, outputs the subspace that the Gaussian is supported on.
Furthermore, the running time of the algorithm is $O(md)$.
\end{theorem}
In this paper, we prove \Theorem{subspace_informal} as an example of our technique from \Section{reduction}.
However, we remark that this is essentially equivalent to other standard frameworks such as propose-test-release.

Next, we show how to learn Gaussians privately, robustly, and efficiently.
An $\alpha$-corrupted sample is a sample that is corrupted by an adversary as follows.
First, we draw $m$ i.i.d.~samples from a Gaussian.
Then, the adversary chooses at most $\alpha m$ samples and modifies them arbitrarily.
In the robust case, we prove the following theorem; more details can be found in \Corollary{robust_gaussian}.
\begin{theorem}[Informal]
\TheoremName{robust_cov_informal}
There is a robust and $(\varepsilon,\delta)$-DP algorithm which given an $\alpha$-corrupted sample of size $m = \wtilde{\Omega}\left( d^{3.5} \ln(1/\delta) / \eps\alpha^3 \right)$
from $\cN(\mu, \Sigma)$ with unknown $\mu \in \bR^d$ and non-singular $\Sigma \in \bR^{d \times d}$, outputs $\wmu$ and $\wSigma$ such that the total variation distance
between $\cN(\mu, \Sigma)$ and $\cN(\wmu, \wSigma)$ is at $O(\alpha \ln(1/\alpha))$ with high probability.
Furthermore, the running time of the algorithm is $\poly(m)$.
\end{theorem}


Important ingredients of the above theorem are two subroutines for learning the mean and the covariance robustly. The following result for mean estimation is proved naturally in our framework, but we note \cite[Theorem 7]{liu2021robust} has an improved sample complexity of $\wtilde{\Omega}(d/\alpha^2 + d^{3/2} \ln(1/\delta) / \eps \alpha)$ for the same problem. Our paper is the first to solve the more difficult task of private, robust, and efficient covariance estimation. The discussion of robust learning of Gaussians can be found in \Section{covariance_robust} and \Section{mean_robust}.

\begin{theorem}[Informal]
\TheoremName{robust_mean_informal}
There is a robust and $(\varepsilon,\delta)$-DP algorithm which given an $\alpha$-corrupted sample of size $m = \wtilde{\Omega}\left( \frac{d^{3/2} \ln(1/\delta) }{\eps\alpha^2} \right)$
from $\cN(\mu, \Sigma)$ with unknown $\mu\in \bR^d$ and $0.5 I_d \preceq \Sigma \preceq 2I_d$, outputs $\wtilde{\mu}$ such that  $\|\mu-\wtilde{\mu} \|_2=O(\alpha\sqrt{\log(1/\alpha)})$ with high probability.
Furthermore, the running time of the algorithm is $\poly(m)$.
\end{theorem}


\hassan{update the above sentence to reflect their new result}
\chris{Do we still want the robust mean estimation or should we use Sewoong's result? Their sample complexity is slightly better than ours ($d/\alpha^2 + d^{3/2} \ln(1/\delta) / \eps \alpha$).}

\subsection{Techniques}
\chris{Please read.}

We present a simple, efficient, and fairly general framework for private statistical estimation which is inspired by the Propose-Test-Release framework \cite{DworkLei09} and the Subsample-And-Aggregate framework~\cite{NissimRS07}.
At a high-level, our framework works as follows.
Suppose that we had a \emph{non-private} estimator $\cA$.
Given a set of samples, we split the set into disjoint subsets and run $\cA$ on each of the subsets (in the spirit of Subsample-And-Aggregate).
We then privately check whether most of the (non-private) outcomes are ``close'' to each other (in the spirit of Propose-Test-Release).
If so, we then aggregate the ``good'' results by compute a (clever) weighted average of the ``good'' results and then releasing an estimate by noising the computed average.
The key is that the ``good'' results ought to be quite structured (for example, coming from a small ball) and averaging helps by providing some additional stability.
This, in turns allows us to apply only a small amount of noise.
We remark that the aggregation component of our framework can be seen as a special case of the
FriendlyCore framework \cite{tsfadia2021friendlycore} which was developed independently to our work (in particular, see Algorithm 4.1 in \cite{tsfadia2021friendlycore}).
One minor difference is that their aggregation is randomized while our aggregation uses a weighted average.


In order for the above framework to work, we need two key properties.
The first is that if we run $\cA$ on two disjoint i.i.d.~samples and the outputs of both are close to the true answer (for example, the true mean or true covariance matrix)
then the outputs are close to each other.
This would allow us to define ``good'' as simply outputs of $\cA$ which are close to other outputs of $\cA$.
The second property we require is that if two weighted averages are computed using nearly the same weights then the weighted average are close to each other.
This allows us to assert that the output is stable under small perturbations of the input, and, thus allows us to add a small amount of noise to the weighted average.
Both of these properties are satisfied for normed spaces (such as the $\ell_2$ space).
Unfortunately, the space that we require for covariance matrices is \emph{not} a normed space (it does not even satisfy the triangle inequality).
However, it does satisfy an \emph{approximate} triangle inequality for points that are sufficiently close together.
In \Definition{semimetric}, we define the notion of a \emph{convex semimetric} space, which, intuitively, slightly relaxes the requirement of being a norm so that it
is general enough to handle other spaces such as the space of covariance matrices (with an appropriate ``norm'').


The final part of the reduction is to output a ``noisy'' version of the average.
Note that at this point, we are essentially guaranteed that the average is quite stable.
In particular, changing a single input point results in very small changes in the average.
In \Definition{masking}, we define the notion of a \emph{masking mechanism}.
At a high-level, a masking mechanism $\cB$ is a mechanism such that $\cB(Y_1), \cB(Y_2)$ are indistinguishable provided $Y_1, Y_2$ are sufficiently close.
Note that a masking mechanism, in and of itself, is not a differentially private algorithm since it does not necessarily ensure that its inputs are close to each other.
For example, the Laplace and Gaussian mechanisms are masking mechanisms which are often used to design differentially private algorithms.
A key technical ingredient for learning covariance matrices is to design an efficient masking mechanism for the space of covariance matrices.
We do this by viewing the input covariance matrix as a vector of length $d^2$ and applying the Gaussian mechanism but where the noise used for the Gaussian mechanism
is scaled by the input covariance matrix itself.
A similar idea of using an empirically rescaled Gaussian mechanism was also utilized by \cite{BGSUZ21} to design an algorithm for learning the mean under Mahalanobis distance
without estimating the covariance matrix.

\subsection{Discussion of a concurrent result}
In an independent work, Kamath, Mouzakis, Singhal, Steinke, and Ullman~\cite{KMSTU21} propose a different polynomial time algorithm for privately learning unbounded high-dimensional Gaussian distributions. At a high-level, they devise an iterative preconditioning method for the covariance matrix (in the spirit of \cite{kamath2019privately, biswas2020coinpress}). To resolve the challenge of handling covariance matrices with unbounded condition numbers, they propose a subspace learning method inspired by the approach of~\cite{SS21}. The iterative approach to bound the ratio of consecutive eigenvalues incurs additional cost of DP composition, resulting in a term in the sample complexity of order $O(d^{5/2}/\varepsilon)$ rather than $O(d^{2}/\varepsilon)$.

In contrast, our approach for preconditioning the covariance matrix is not iterative, and in its core relies on a masking/noising mechanism for covariance matrices.
In particular, we are able to compute a preconditioner with roughly $\wtilde{O}(d^{2} / \eps)$ samples.
Non-privately, $O(d)$ samples are needed to compute a preconditioner.
It is an interesting open problem to see if $O(d)$ samples suffices for the private setting as well.
Our general reduction also allows us to apply the result in other settings such as \emph{robust} learning of Gaussians in polynomial time.

\modify{
In another independent work, Kothari, Manurangsi, Velingker \cite{KMV21} proposed an efficient algorithm for robust and differentially private estimation
of the mean, covariance matrix, and higher moments of distributions that satisfy either a certifiable subgaussianity or certifiable hypercontractivity.
Their algorithm is designed using the sum-of-squares paradigm.
In the case of Gaussian distributions, their work obtains an efficient robust and $(\eps, \delta)$-DP algorithm for learning a Gaussian in $\bR^d$ with
sample complexity $\wtilde{\Omega}(d^8 \ln^4(1/\delta) / \alpha^4 \eps^4)$ where $\alpha$ is the corruption parameter.
}

In yet another independent work, Tsfadia, Cohen, Kaplan, Mansour, and Stemmer \cite{tsfadia2021friendlycore} introduced the FriendlyCore framework which allows for differentially private aggregation.
The aggregation component of our algorithm is quite similar to that of FriendlyCore (see \cite[Algorithm 4.3]{tsfadia2021friendlycore}) with a minor difference being that we compute a weighted average instead of computing a random ``core''.
Indeed, after the first version of our paper was posted on the arXiv, Tsfadia, Cohen, Kaplan, Mansour, and Stemmer \cite{tsfadia2021friendlycore} showed how their framework can also be used to learn the covariance matrix using the tools developed in this paper (in particular, the masking mechanism for covariance matrices in \Lemma{CovarianceTechnical}).
For more details, see Appendix~A of \cite{tsfadia2021friendlycore}.

\subsection{More on related work}
\hassan{add a citation to \cite{diakonikolas2017being} and maybe to \cite{cheng2019high}. We don't have a citation to \cite{dong2019quantum} or \cite{diakonikolas2016robust, lai2016agnostic} in this section either}

\hassan{add citations here and/or in the reduction to FriendlyCore~\cite{tsfadia2021friendlycore} and perhaps to the older GoodCenter~\cite{nissim2016locating}. Regarding the aggregate approaches, perhaps we should cite this one too \cite{smith2011privacy}?}
Most relevant to our work are the results on learning Gaussians with respect to the TV distance under $(\varepsilon, \delta)$-DP. Karwa and Vadhan~\cite{karwa2018finite} established a polynomial-time and sample-efficient method for learning unbounded univariate Gaussians under $(\varepsilon, \delta)$-DP. One can use this method for learning multivariate Gaussians as well, by simply learning the one dimensional projection of the Gaussians along each axis (and invoking the composition for DP). While this approach works well for axis-aligned Gaussians, for general Gaussians its sample complexity depends on the condition number of the covariance matrix.
Kamath, Li, Singhal, and Ullman~\cite{kamath2019privately} consider the problem of privately learning high-dimensional distributions. They provide a polynomial time algorithm for learning high-dimensional Gaussians under the (stronger notion of) concentrated DP~\cite{dwork2016concentrated, bun2016concentrated}. The sample complexity of this approach improves over the multivariate version of \cite{karwa2018finite}; however, it still depends (logarithmically) on the condition number of covariance matrix. On a follow up work, Biswas, Dong, Kamath, and Ullman~\cite{biswas2020coinpress} propose the more practical method of CoinPress whose sample complexity matches that of~\cite{kamath2019privately} for learning multivariate Gaussians. CoinPress is shown to empirically work better than its counterparts---namely the multivariate versions of~\cite{karwa2018finite, du2020differentially}---in high dimensions.
    
The best known lower bounds for learning $d$-dimensional Gaussians up to TV distance $\alpha$ under $(\varepsilon, \delta)$-DP is $\Tilde{\Omega}(d/\alpha \varepsilon)$~\cite{kamath2019privately}. In contrast, the sample complexity of all the methods described above depends on the condition number of the covariance matrix. Utilizing the private hypothesis selection framework proposed in~\cite{bun2019private}, Aden-Ali, Ashtiani, and Kamath~\cite{aden2021sample} proved the first sample complexity bound for privately learning general Gaussians that does not depend on the condition number or the size of the parameters. The proof is, however, non-constructive, as it invokes Zorn's lemma to prove the existence of a locally small cover for covariance matrices. 
    
In the non-private setting, an $\alpha$-accurate estimation of the mean of a $d$-dimensional Gaussian distribution with respect to the Mahalanobis distance can be found using $O(d/\alpha^2)$ samples. Achieving linear dependence on $d$, however, is not easy in the private setting. Brown, Gaboardi, Smith, Ullman, and Zakynthinou~\cite{BGSUZ21} propose a method for solving this problem under $(\varepsilon, \delta)$-DP with linear dependence on $d$. Their method is, however, not computationally efficient. One interesting aspect of~\cite{BGSUZ21} is the use of a variant of the Gaussian mechanism that is scaled with the empirical covariance.
We also make use of a similar idea to noise the covariance matrix.

A related line of work is the problem of private mean estimation with respect to the Euclidean distance. It is noteworthy that these results do not translate to learning Gaussians with respect to the TV distance. Kamath, Singhal, and Ullman~\cite{kamath2020private} prove tight minimax rates for mean estimation for distributions with bounded $k$-th moment. Liu, Kong, Kakade, and Oh~\cite{liu2021robust} proposed a robust, efficient, and private mean estimation method for sub-Gaussian distributions with bounded mean.

Other related work include ``universal'' private estimators for mean, variance, and scale under pure differential privacy~\cite{dong2021universal}, private Principal Component Analysis~\cite{chaudhuri2013near, wei2016analysis,dong2022differentially, liu2022dp} and private learning of Gaussian mixture models~\cite{aden2021privately, kamath2020differentially}. We refer the readers to \cite{dwork2014algorithmic, vadhan2017complexity} for an introduction to differential privacy, and to \cite{kamath2020primer} for a survey on private statistical learning.

\modify{
At a very high-level, the algorithms in this paper all follow the same approach of privately aggregating the outputs of non-private algorithms.
Our approach can be viewed as combining the Subsample-And-Aggregate paradigm with an aggregation scheme that is is similar to the FriendlyCore framework of \cite{tsfadia2021friendlycore}.
Older works have also made use of this blueprint including \cite{nissim2016locating} in the context of clustering and \cite{smith2011privacy} in the context of statistical estimation.
}

%% file: prelim.tex
\section{Preliminaries}
\SectionName{prelim}
For a vector $v \in \bR^d$, we use $\|v\|_2$ to denote the Euclidean norm of the vector $v$.
For a matrix $A \in \bR^{d \times d}$, we use $\|A\|_F$ to denote the Frobenius norm of $A$ and $\|A\|$ to denote the operator norm of $A$.
A matrix $A$ is positive semi-definite (resp.~positive definite),
denoted $A \succeq 0$ (resp.~$A \succ 0$) if $A$ is symmetric and all its eigenvalues are non-negative (resp.~positive).
If $A \succeq 0$ then $A^{1/2}$ denotes the unique positive semi-definite matrix satisfying $A = (A^{1/2})^2$.
Further, if $A \succ 0$ then $A^{-1/2} = (A^{1/2})^{-1}$.
In this paper, we also write $\cS^d$ to denote the positive-definite cone in $\bR^{d \times d}$.

\subsection{Probability Facts}
In this paper, we use $\dtv{\cD_1}{\cD_2}$ and $\KL{\cD_1}{\cD_2}$ to denote the total variation distance and KL divergence
between two distributions $\cD_1, \cD_2$.

\begin{fact}
	\FactName{KLNormal}
    Let $\mu_1, \mu_2 \in \bR^d$ and $\Sigma_1, \Sigma_2 \succ 0$. Then
	\begin{align*}
		\KL{\cN(\mu_1, \Sigma_1)}{\cN(\mu_2, \Sigma_2)}
		& = \frac{1}{2} \big[ \Tr(\Sigma_2^{-1} \Sigma_1 - I) + (\mu_2 - \mu_1)\transpose \Sigma_2^{-1} (\mu_2 - \mu_1) - \ln \det(\Sigma_2^{-1} \Sigma_1) \big].
	\end{align*}
	Moreover, suppose that all the eigenvalues of $\Sigma_2^{-1} \Sigma_1$ are at least $\frac{1}{2}$.
	Then
	\begin{align*}
        \KL{\cN(\mu_1, \Sigma_1)}{\cN(\mu_2, \Sigma_2)} & \leq \frac{1}{2} \big[ \|\Sigma_{2}^{-1/2} \Sigma_1 \Sigma_{2}^{-1/2} - I\|_F^2 + (\mu_2 - \mu_1)\transpose \Sigma_2^{-1} (\mu_2 - \mu_1) \big]
	\end{align*}
\end{fact}
The first assertion of \Fact{KLNormal} is a well-known formula for the KL divergence between Normal distributions.
The proof of the second assertion can be found in \Appendix{KLNormal}.

\begin{lemma}[Pinsker's Inequality \protect{\cite[Lemma 2.5]{Tsybakov09}}]
    \LemmaName{Pinsker}
    For any two distributions $\cD_1, \cD_2$, $\dtv{\cD_1}{\cD_2}^2 \leq \frac{1}{2} \KL{\cD_1}{\cD_2}$.
\end{lemma}

The following lemma is a standard bound on the (upper) tails of a Gaussian.
\begin{lemma}[\protect{e.g.~\cite[Proposition~2.1.2]{Ver18}}]
    \LemmaName{gaussConcentration}
    Let $g \sim \cN(0, \sigma^2)$.
    For all $x > 0$, $\prob{g \geq \sigma x} \leq e^{-x^2/2}$.
\end{lemma}

We require the following concentration inequality which can be easily derived from well-known techniques (see \cite[\S 2.1.3]{Wainwright19}).
For completeness, we provide a proof in \Appendix{chiSquareConcentration}.
\begin{lemma}
    \LemmaName{chiSquareConcentration}
    Let $g \sim \cN(0, I_d)$ and $A$ be a symmetric matrix. 
    Then for all $x > 0$,
    \[
        \prob{g\transpose A g \geq \Tr(A) + 4 \|A\|_F \cdot \sqrt{x} + 4\|A\|\cdot x} \leq \exp(-x).
    \]
\end{lemma}

\subsection{Differential Privacy}
In this section, we let $\cX, \cY$ denote sets and $n \in \bN$.
Two datasets $D = (X_1, \ldots, X_n), D' = (X_1, \ldots, X_n) \in \cX^n$ are said to be \emph{neighbouring} if $d_H(D, D') \leq 1$
where $d_H$ denotes Hamming distance, i.e.~$d_H(D, D') = |\{ i \in [d] \,:\, X_i \neq X_i' \}|$.
\begin{definition}[\protect{\cite{DMNS06, dwork2006our}}]
    A mechanism $\cM \colon \cX^n \to \cY$ is said to be $(\eps, \delta)$-differentially private if for all neighbouring datasets $D, D' \in \cX^n$
    and all $Y \subseteq \cY$, we have
    \[
        \prob{\cM(D) \in Y} \leq e^{\eps} \cdot \prob{\cM(D') \in Y} + \delta.
    \]
\end{definition}
In this paper, we use DP as shorthand for differentially private or differential privacy, depending on context.
\begin{remark}
    The astute reading may notice that the algorithms we present in this paper are \emph{not} invariant under permutation of the input.
    Given this, it may seem more desirable to define differential privacy where we view datasets as unordered sets instead of ordered tuples.
    However, one can show that if one has an $(\eps, \delta)$-DP algorithm for ordered tuples than one also has a $(\eps, \delta)$-DP for unordered sets
    by randomly permuting the input \cite[Proposition~D.6]{BGSUZ21}.
\end{remark}
\begin{definition}
    \DefinitionName{LogLikelihoodRatio}
    Let $\cD_1, \cD_2$ be two continuous distributions defined on $\bR^d$ and let $f_1, f_2$ be the respective density functions.
    We use $\privLoss{\cD_1}{\cD_2} \colon \bR^d \to \bR$ to denote the logarithm of the likelihood ratio, i.e.~for any $x \in \bR^d$,
    \begin{equation}
        \privLoss{\cD_1}{\cD_2}(x) \coloneqq \ln \frac{f_1(x)}{f_2(x)}.
    \end{equation}
\end{definition}

\begin{remark}
    Note that $\expects{Y \sim \cD_1}{\privLoss{\cD_1}{\cD_2}(Y)} = \KL{\cD_1}{\cD_2}$.
\end{remark}

\begin{definition}
    \DefinitionName{indistinguishable}
    Let $\cD_1, \cD_2$ be two distributions defined on a set $\cX$.
    Then $\cD_1, \cD_2$ are said to be $(\eps, \delta)$-indistinguishable if for all measurable $S \subseteq \cX$, we have
    \[
        \probs{Y \sim \cD_1}{Y \in S} \leq e^{\eps} \probs{Y \sim \cD_2}{Y \in S} + \delta
        \quad \text{and} \quad
        \probs{Y \sim \cD_2}{Y \in S} \leq e^{\eps} \probs{Y \sim \cD_1}{Y \in S} + \delta.
    \]
\end{definition}
In light of \Definition{indistinguishable}, a mechanism $\cM$ is $(\eps, \delta)$-DP if and only if for every neighbouring datasets
$D, D'$, the distributions given by $\cM(D)$ and $\cM(D')$ are $(\eps, \delta)$-indistinguishable.

In this paper, we make use of the following sufficient condition for indistinguishability;
a proof can be found in \Appendix{privLossDP}.
\begin{lemma}
    \LemmaName{privLossDP}
    Let $\cD_1, \cD_2$ be continuous distributions defined on $\bR^d$.
    If
    \[
        \probs{Y \sim \cD_1}{\privLoss{\cD_1}{\cD_2}(Y) \geq \eps} \leq \delta
        \quad \text{and} \quad
        \probs{Y \sim \cD_2}{\privLoss{\cD_2}{\cD_1}(Y) \geq \eps} \leq \delta
    \]
    then $\cD_1, \cD_2$ are $(\eps, \delta)$-indistinguishable.
\end{lemma}

The following claim follows from the probability density function of Normal distributions.
\begin{claim}
    \ClaimName{privLossNormal}
    Let $\mu_1, \mu_2 \in \bR^d$ and $\Sigma_1, \Sigma_2 \succ 0$.
    Then
    \begin{equation}
        \EquationName{privLossNormal}
        \privLoss{\cN(\mu_1, \Sigma_1)}{\cN(\mu_2, \Sigma_2)}(x)
        =
        \ln\left(\frac
        {\det(\Sigma_1)^{-1/2} \exp\left( -\frac{1}{2}(x-\mu_1)\transpose \Sigma_1^{-1} (x-\mu_1) \right)}
        {\det(\Sigma_2)^{-1/2} \exp\left( -\frac{1}{2}(x-\mu_2)\transpose \Sigma_2^{-1} (x-\mu_2) \right)}
        \right).
    \end{equation}
\end{claim}

In this paper, 
we make use of the truncated Laplace distribution,
denoted by $\TLap(\Delta, \eps, \delta)$ whose probability density function is given by
\[
    f_{\TLap(\Delta, \eps, \delta)}(x) \coloneqq
    \begin{cases}
        Be^{-|x|/\lambda} & x \in [-A, A] \\
        0                 & x \notin [-A, A]
    \end{cases},
\]
where $\lambda = \frac{\Delta}{\eps}$, $A = \frac{\Delta}{\eps} \ln\left( 1 + \frac{e^{\eps} - 1}{2\delta} \right)$, $B = \frac{1}{2\lambda(1-e^{-A/\lambda})}$.

\begin{theorem}[\protect{\cite[Theorem 1]{GDGK20}}]
    \TheoremName{TLap}
    Suppose that $q \colon \cX \to \bR$ is a function with sensitivity $\Delta$.
    Then the mechanism $q(x) + Y$ where $Y \sim \TLap(\Delta, \eps, \delta)$ is $(\eps, \delta)$-DP.
\end{theorem}

%% file: reduction.tex
\section{A framework for private statistical estimation}
\SectionName{reduction}

In this section, we describe an efficient and fairly general reduction from private statistical estimation to the non-private counterpart.
Our reduction is inspired by the Propose-Test-Release framework \cite{DworkLei09} and the Subsample-And-Aggregate framework \cite{NissimRS07}.
At a very high level, we partition the data set into multiple disjoint subsets and run the non-private estimation method on each subset. We then privately check whether most of these solutions are ``close'' to each-other. If this is not the case then the method fails. Otherwise, we compute a weighted average of the ``concentrated'' solutions.
Roughly speaking, solutions which are ``close'' to many other candidate solutions (say $0.7$ fraction of them) are given weight $1$ while solutions that are ``close'' to fewer candidate solutions (say less than $0.6$ fraction of them) are given weight $0$.
To help with sensitivity, we then linearly interpolate the weight for solutions that are in the middle.
Finally, we output a ``noisy/masked'' version of this average. In order for the reduction to work in rather general scenarios (such as learning the high dimensional Gaussians with respect the TV distance), we first define the notion of convex semimetric spaces.
This helps to formalize the notion of ``closeness'' between solutions as well as the weighted average over a number of solutions.

\subsection{Convex semimetric spaces and masking mechanisms}
In this paper, we work with the following notion of a convex semimetric space.
The key property to keep in mind is that for semimetric spaces, we only have an \emph{approximate} triangle inequality,
as long as the points are significantly close together.
\begin{definition}[Convex semimetric space]
\DefinitionName{semimetric}
Let $\cY$ be a convex set and let $\dist \colon \cY \times \cY \to \bR_{\geq 0}$.
We say $(\cY, \dist)$ is a convex semimetric space if there exists absolute constants $t\geq 1$, $\phi\geq 0$, and $r>0$ such that for every $k\in \bN$ and every $Y, Y_1, Y_2, \ldots, Y_k \in \cY$, the following conditions hold.
\begin{enumerate}
    \item $\dist(Y, Y) = 0$ and $\dist(Y_1, Y_2)\geq 0$.
    \item \textbf{Symmetry.} $\dist(Y_1, Y_2)=\dist(Y_2, Y_1)$.
    \item \textbf{$t$-approximate $r$-restricted triangle inequality.} If $\dist(Y_1, Y_2), \dist(Y_2, Y_3) \leq r$ then $\dist(Y_1, Y_3) \leq t \cdot (\dist(Y_1, Y_2) + \dist(Y_2, Y_3))$.
    \item \textbf{Convexity.} For all $\alpha \in \Delta_k$,
    $\dist\left(\sum \alpha_iY_i, Y\right) \leq  \sum \alpha_i \dist(Y_i,Y)$.
    \item \textbf{$\phi$-Locality.}  For all $\alpha,\alpha' \in \Delta_k$, ${\dist\left(\sum_i {\alpha_i}Y_i, \sum_i {\alpha_i'Y_i}\right) \leq \sum_i |\alpha_i-\alpha_i'| (\phi+\max_{i,j}\left(\dist(Y_i, Y_j)\right)}$

\end{enumerate}
where $\Delta_k$ denotes the $k$-dimensional probability simplex. When $r$ is unspecified we take it to mean that $r = \infty$ and refer to it as a $t$-approximate triangle inequality.
\end{definition}
The following technical lemma (whose proof appears in \Appendix{reduction}) is helpful for learning covariance matrices.
\begin{lemma}
    \LemmaName{semimetric_cov}
    Let $\cS^d$ be the set of all $d \times d$ positive definite matrices.
    For $A, B \in \cS^d$ let $\dist(A, B) = \max\{\| A^{-1/2} B A^{-1/2} - I \|, \|B^{-1/2} A B^{-1/2} - I\|\}$.
    Then $(\cS^d, \dist)$ is a convex semimetric which satisfies a $(3/2)$-approximate $1$-restricted triangle inequality and $1$-locality.
\end{lemma}


\begin{definition}[Masking mechanism]
\DefinitionName{masking}
Let $(\cY, \dist)$ be a convex semimetric space.
A randomized function $\cB \colon \cY \to \cY$ is a $(\gamma, \eps, \delta)$-masking mechanism for $(\cY, \dist)$
if for all $Y_1, Y_2 \in \cY$ satisfying $\dist(Y_1, Y_2) \leq \gamma$, we have that $\cB(Y_1), \cB(Y_2)$ are $(\eps, \delta)$-indistinguishable.
\end{definition}

\begin{remark}
It is straightforward to show that for $p \geq 1$,
the vector space $\bR^d$ equipped with $\dist(Y_1,Y_2)=\|Y_1-Y_2\|_p$ forms a convex semimetric space (see \Proposition{semimetric_normed}).
In particular, the Laplace mechanism~\cite{DMNS06} is a masking mechanism for $(\bR^d,\ell_1)$ and
the Gaussian mechanism is a masking mechanism for $(\bR^d,\ell_2)$.
In \Section{covariance}, we design a masking mechanism for the set of covariance matrices with respect to
$\dist(A, B) = \max\{\| A^{-1/2} B A^{-1/2} - I \|, \|B^{-1/2} A B^{-1/2} - I\|\}$.
\end{remark}

\begin{definition}[Concentration of a masking mechanism]
Let $(\cY, \dist)$ be a convex semimetric space and  $\cB$ be $(\gamma, \varepsilon, \delta)$-masking mechanism for it.
We say $\cB$ is $(\alpha, \beta)$ concentrated if for all $Y\in \cY$, $\prob{\dist(\cB(Y), Y)> \alpha}\leq \beta$. 
\end{definition}

\subsection{The Private Populous Mean Estimator}
In this section we outline an efficient and fairly general reduction from private statistical estimation to the non-private counterpart. We use $\cA \colon \cX^* \to \cY$ to denote a generic (efficient) non-private mechanism which takes a collection of samples
as inputs and outputs an element in $\cY$. We assume that $(\cY, \dist)$ is a convex semimetric space, and we have access to (an efficient) masking mechanism for $(\cY, \dist)$. Intuitively, the error of the estimator is measured by $\dist(\hat{Y},Y^*)$, where $\hat{Y}$ is our estimation while $Y^*$ is the ground truth.

At a high-level, the algorithm works as follows.
We split the dataset into $k$ disjoint groups and run the non-private algorithm $\cA$ on each group $i$ to obtain an output $Y_i$.
We then consider a ball of radius $r / t$ around $Y_i$ with respect to $\dist$\footnote{$r$ and $t$ are chosen in a way that $(\cY, \dist)$ satisfies $r$-restricted $t$-approximate triangle inequality.} and compute $q_i$, the fraction of $Y_j$'s that are inside this ball. We then look at the average value of the $q_i$, i.e.,~$Q = \frac{1}{k} \sum_{i \in [k]} q_i$. Intuitively, $Q$ efficiently measures the ``stability'' of $\cA$ across the data set partitions.
Namely, $Q$ is large when most of the $Y_i$'s are concentrated in a small ``popular region'', and small when $Y_i$'s are scattered around $\cY$.
If $Q$ is small, then $\cA$ is not sufficiently stable and we can announce failure privately using a noisy threshold.
If $Q$ is large, then we take a weighted average of $Y_i$'s that are in the popular region and then output a ``masked'' version of this average. The following lemma establishes the privacy and accuracy of \Algorithm{reduction} and its proof can be found in \Appendix{reduction_proof}.

\begin{algorithm}
    \caption{Private to non-private reduction via populous mean estimation}
    \AlgorithmName{reduction}
    \textbf{Input:} Dataset $D = (X_1, \ldots, X_m)$; parameters $\eps, \delta, r, \phi > 0$; $k \leq m$; $t \geq 2$.
    \begin{algorithmic}[1]
        \State Let $s \gets \lfloor m / k \rfloor$.
        \State For $i \in [k]$, let $Y_i \gets \cA(\{X_\ell\}_{\ell = (i-1)s+1}^{is})$. \LineName{ReductionNonPrivate}
        \State For $i \in [k]$, let $q_i \gets \frac{1}{k} |\{ j \in [k] \,:\, \dist(Y_i, Y_j) \leq r / t \}|$.
        \State Let $Q \gets \frac{1}{k} \sum_{i \in [k]} q_i$. \LineName{ReductionQ}
        \State Let $Z \sim \TLap(2/k, \eps, \delta)$.
        \State Let $\what{Q} \gets Q + Z$.
        \State If $\widehat{Q} < 0.8 + \frac{2}{k\eps} \ln\left( 1 + \frac{e^{\eps} - 1}{2\delta} \right)$, fail and return $\perp$. \LineName{ReductionTLap}
        \State For $i \in [k]$, let $w_i=\min\left(1, 10\max(0,q_i-0.6)\right)$ \LineName{WeightDef}
        \State $\mu=(\sum_{i=1}^{k}w_i Y_i)/(\sum_{i=1}^{k}w_i)$ \LineName{WeightedAverage}
        \State Return $\cB(\mu)$.
    \end{algorithmic}
\end{algorithm}

\begin{lemma}
    \LemmaName{new_reduction}
    Let $(\cY, \dist)$ be a convex semimetric space that satisfies $r$-restricted $t$-approximate triangle inequality and $\phi$-locality
    for some $r>0, t\geq 1, \phi\geq0$.
    Let $\cB$ be a $(400(r+\phi)/k,\varepsilon, \delta)$-masking mechanism with respect to $(\cY, \dist)$.
    Then
    \begin{enumerate}
        \item {\bf Privacy.} For $k\geq 140$, \Algorithm{reduction} is $(2\eps, 4e^{\eps} \delta)$-DP.
        \item {\bf Utility.} Assume $k \geq \frac{20}{\eps} \ln\left( 1 + \frac{e^{\eps} - 1}{2\delta} \right)$. For $\alpha_1\leq r/t$ assume that there exist $Y^*\in \cY$ such that $\dist(Y_i, Y^*) \leq \alpha_1/t$ for all $i\in [k]$. 
        If $\cB$ is ($\alpha_2/t, \beta$)-concentrated then the output $\hat{Y}$ of \Algorithm{reduction} satisfies
        \[
        \prob{\dist(\hat{Y}, Y^*) > \alpha_1 + \alpha_2}\leq\beta
        \]
    \end{enumerate}
\end{lemma}

\begin{remark}
    For privacy, \Lemma{new_reduction} requires that $k$ is larger than some constant. If this is not the case, we can always add ``dummy'' points into the dataset.
\end{remark}

\begin{remark}
The computational complexity of \Algorithm{reduction} is dominated by these computations: \begin{itemize}
    \item $k$ queries to the oracle $\cA$ (with inputs of size $m/k$).
    \item $k^2$ computations of the $\dist(,.,)$ function.
    \item One query to the masking mechanism $\cB$.
\end{itemize}
In particular, if $\cA$, $\cB$, and the algorithm for computing $\dist$ all run in polynomial time then \Algorithm{reduction} runs in polynomial time.
\end{remark}

%% file: subspace.tex
\section{Learning the subspace}
As a simple application of our technique, we consider the problem of privately learning a subspace.
More specifically, if the samples are drawn from $\cN(0, \Sigma)$ then the goal is to learn the range of $\Sigma$.
Recently, Singhal and Steinke \cite{SS21} also studied the problem of privately learning a subspace albeit under slightly different assumption than ours.
For learning a subspace \emph{exactly}, they gave an algorithm that was sample optimal but not computationally efficient.
For learning a subspace \emph{approximately}, they provided an algorithm that was sample efficient and computationally efficient.
The latter is used (with some modifications) in \cite{KMSTU21} to efficiently learn the covariance matrix of a Gaussian distribution.
In this section, we give a computationally-efficient algorithm for learning a subspace \emph{exactly} with a slightly different sample complexity
(albeit with different assumptions).
Formally, we prove the following theorem.

\begin{theorem}
    \TheoremName{subspace}
    There exists a mechanism $\cM$ satisfying the following.
    For all $m, d \in \bN$, $\eps, \delta \in (0,1)$, and dataset $D = (X_1, \ldots, X_m)$:
    \begin{enumerate}[noitemsep, topsep=0pt]
        \item Given $\eps, \delta, D$ as input, $\cM$ is $(2\eps, 4e^{\eps}\delta)$-DP and runs in time $O(md)$.
        \item If $\Sigma \succeq 0$, $X_1, \ldots, X_m \stackrel{\text{i.i.d.}}{\sim} \cN(0, \Sigma)$, and $m \geq \max\{140, \lceil \frac{20}{\eps} \ln(1 + \frac{e^{\eps} - 1}{2\delta}) \rceil \} \cdot d = O(d\log(1/\delta) / \eps)$
        then, almost surely, $\cM$ outputs the orthogonal projection matrix onto the range of $\Sigma$.
    \end{enumerate}
\end{theorem}

The mechanism that achieves \Theorem{subspace} appears in \Algorithm{subspace} and is designed using the template given in \Algorithm{reduction}.
In this section, we define
\[
    \dist(A, B) =
    \begin{cases}
        0 & \text{if $A = B$ are orthogonal projection matrices} \\
        \infty & \text{otherwise}
    \end{cases},
\]
where $A, B$ are two matrices in $\bR^{d \times d}$.
Note that this makes the choice of $r, t, \phi$ moot since the distance between any two matrices is either $0$ or $\infty$.
Finally, recall that orthogonal projection matrices are unique.
In other words, if $P_1$ and $P_2$ are both orthogonal projection matrices onto a subspace $S$ then $P_1 = P_2$.
\begin{algorithm}
    \caption{A private algorithm for learning the subspace.}
    \AlgorithmName{subspace}
    \textbf{Input:} Dataset $D = (X_1, \ldots, X_m)$; parameters $\eps, \delta \in (0,1)$.
    \begin{algorithmic}[1]
        \LineComment{{\footnotesize Parameter settings.}}
        \State $r \gets 1$, $t \gets 1$, $\phi \gets 1$. \Comment{{\footnotesize Not used in \Algorithm{reduction}.}}
        \State $k \gets \max\{30, \frac{20}{\eps} \ln\left( 1 + \frac{e^{\eps} - 1}{2\delta} \right)\}$
        \vspace{1em}
        \Function{$\cA$}{$X_1, \ldots, X_s$}
        \State Let $P$ be orthogonal projection matrix onto $\Span\{X_1, \ldots, X_s\}$.
        \State \textbf{Return} $P$.
        \EndFunction
        \vspace{1em}
        \Function{$\cB$}{$P$}
            \State \textbf{Return} $P$.
        \EndFunction
    \end{algorithmic}
\end{algorithm}
\begin{proof}[Proof of \Theorem{subspace}]
Using the convention $\infty \leq \infty$, it is easy to check that the $\dist$ function we defined above forms a convex semimetric space for matrices (see \Definition{semimetric}). Also, the identity function is a $(\gamma,\varepsilon, \delta)$-masking mechanism for all $\gamma, \varepsilon, \delta \in (0,1)$ since for all $A, B$, $\dist(A, B) < \infty$ only if $A, B$ are both orthogonal projection matrices and $A = B$.
Therefore, the privacy statement follows from \Lemma{new_reduction}.

For the utility, note that the masking mechanism is the identity function, so it is $(0,0)$-concentrated.
Since we set $k = \max\left\{ 140, \left \lceil \frac{20}{\eps} \ln\left( 1 + \frac{e^{\eps} - 1}{2\delta} \right) \right \rceil \right\}$, our assumption on $m$ implies that $m \geq kd$.
In particular, $s \geq d$ (see \Algorithm{reduction}).
Therefore, with probability $1$, the output of $\cA(\{X_{\ell}\}_{\ell=(i-1)s+1}^{is})$ is the orthogonal projection matrix onto the range of $\Sigma$ for all $i \in [k]$ almost surely.
Therefore, the utility conditions of \Lemma{new_reduction} are satisfied with $\alpha_1=\alpha_2=\beta=0$.
\end{proof}

%% file: covariance.tex
\section{An efficient algorithm for learning the covariance matrix}
\SectionName{covariance}

In this section we show how to use the general reduction to learn high-dimensional Gaussian distributions with respect to the total variation distance. Here, the main hurdle is to find an approximation to the covariance matrix for the zero-mean case. Based on \Lemma{semimetric_cov}, the following $\dist$ function forms a semimetric space for positive definite matrices. 

\begin{equation}
    \EquationName{CovarianceDistance}
    \dist(\Sigma_1, \Sigma_2) =
    \begin{cases}
        \max\left(\|\Sigma_2^{-1/2} \Sigma_1 \Sigma_2^{-1/2} - I_d\|, \|\Sigma_1^{-1/2} \Sigma_2 \Sigma_1^{-1/2} - I_d\|\right) & \text{if $\rank \Sigma_1 = \rank \Sigma_2 = d$} \\
        \infty & \text{otherwise}
    \end{cases}.
\end{equation}

The main steps to use the reduction are to define a masking mechanism and show that it is concentrated. We start by stating the main algorithmic idea of this section, which is a masking mechanism with respect to the $\dist$ function in \Equation{CovarianceDistance}.

Before we state our masking mechanism, we provide some intuition on its design.
First, imagine that we were working on $\bR^{d^2}$ with the standard Euclidean metric.
Let $G \sim \cN(0, I_{d^2})$ and consider the Gaussian mechanism $\cB(X) = X + \eta G$.
If $X_1, X_2 \in \bR^{d^2}$ and $\|X_1 - X_2\|_2^2 \leq \gamma^2 d$, where $\gamma \lesssim \eps \eta / \sqrt{d\ln(1/\delta)}$,
then a standard analysis shows that $\cB(X_1), \cB(X_2)$ are $(\eps, \delta)$-indistinguishable.

The masking mechanism in \Lemma{CovarianceTechnical} is attempting to mimic the Gaussian mechanism where we view $d \times d$ matrices as vectors of length $d^2$.
A first attempt may be to take a matrix $\Sigma$, compute $\wSigma = \Sigma^{1/2} + \eta G$, where $G$ is a $d \times d$ matrix with $\cN(0,1)$ entries,
and return $\wSigma \wSigma\transpose$.
However, this does not quite work for covariance estimation since we require the noise to have a difference scale in different directions,
i.e.~the noise along each eigenvector should scale with the corresponding eigenvalue.
Thus, we use an empirically scaled Gaussian mechanism where the noise we add is shaped and scaled by the input itself.
\begin{lemma}
    \LemmaName{CovarianceTechnical}
    For a matrix $\Sigma \succ 0$, define $\cB(\Sigma) = \Sigma^{1/2}(I + \eta G)(I + \eta G)\transpose \Sigma^{1/2}$ where
    $G \in \bR^{d \times d}$ is a $d \times d$ matrix with independent $\cN(0, 1)$ entries.
    
    Then for every $\eta>0$ and $\varepsilon, \delta \in (0,1]$, $\cB$ is a ($\gamma, \varepsilon, \delta$)-masking mechanism with respect to $\dist$ for
    \[
        \gamma = \min\left\{
            \sqrt{\frac{\eps}{2d(d+1/\eta^2)}},
            \frac{\eps}{8d \sqrt{\ln(2/\delta)}},
            \frac{\eps}{8\ln(2/\delta)},
            \frac{\eps \eta}{12 \sqrt{d} \sqrt{\ln(2/\delta)}}
        \right\}.
    \]
\end{lemma}
The proof of \Lemma{CovarianceTechnical} can be found in \Appendix{CovarianceTechnical}. Next, we show that the masking mechanism in \Lemma{CovarianceTechnical} is concentrated.
\begin{lemma}
    \LemmaName{CovarianceAccuracy}
    Let $\Sigma \succ 0$ and set $\eta = \frac{1}{C_1(\sqrt{d} + \sqrt{\ln(4/\beta)})}$ for a sufficiently large (universal) constant $C_1 > 0$.
    Then the masking mechanism $\cB(\Sigma)$ defined in \Lemma{CovarianceTechnical} is $(1/100, \beta/2)$-concentrated.
\end{lemma}
\begin{proof}
    Let $\wSigma =\cB(\Sigma) = \Sigma^{1/2}(I + \eta G)(I + \eta G)\transpose \Sigma^{1/2}$.
    First, we prove that
    \[
        \| \Sigma^{-1/2} \wSigma \Sigma^{-1/2} - I_d \|
        \leq 1/400,
    \]
    with probability $1 - \beta/2$.
    Indeed, we have that, for a sufficiently large constant $C'> 0$,
    \begin{align*}
        \| \Sigma^{-1/2} \wSigma \Sigma^{-1/2} - I_d \|
        & = \| (I_d + \eta G)(I + \eta G)\transpose - I_d \| \\
        & \leq 2 \eta G + \eta^2 \|G G\transpose \| \\
        & \leq 2 \eta C' (\sqrt{d} + \sqrt{\ln(4/\beta)}) + \eta^2 (C'(\sqrt{d} + \sqrt{\ln(4/\beta)}))^2,
    \end{align*}
    where we used the fact that $\|G\| \leq C(\sqrt{d} + \sqrt{\ln(4/\beta)})$ with probability $1-\beta/2$
    (see \Theorem{GaussianSpectralNorm}).
    Taking $\eta \leq \frac{1}{1600 C(\sqrt{d} + \sqrt{\ln(4/\beta)})}$ gives that
    \[
        \| \Sigma^{-1/2} \wSigma \Sigma^{-1/2} - I_d \| \leq 1/400,
    \]
    as desired.
    
    In order to show that $\dist(\Sigma, \wSigma)$, it remains to show that
    \[
        \|\wSigma^{-1/2} \Sigma \wSigma^{-1/2} - I_d \| \leq 1/100.
    \]
    This follows directly from \Lemma{SpectralBounds}.
\end{proof}

\Algorithm{covariance} summarizes the modules needed for \Lemma{new_reduction} to work, including the definitions of the mechanism $\cA, \cB$ as well as the required parameters. Note that $\cA$ returns just the empirical covariance matrix (although we could easily replace it by any algorithm which approximates the covariance matrix). One last piece is to show that the output of the non-private $\cA$ is accurate enough.

\begin{lemma}
\LemmaName{CovarConcentration}
There is a universal constant $C_2$ such that the following holds.
Suppose $s \geq C_2 (d + \ln(4/\beta))$.
Let $X_1, \ldots, X_s \stackrel{\text{i.i.d.}}{\sim} \cN(0, \Sigma)$
and $\wSigma = \sum_{\ell=1}^s X_\ell X_\ell\transpose$.
Then $\dist(\Sigma, \wSigma) \leq 1/100$ with probability $1-\beta/2$.
\end{lemma}
\begin{proof}
    Applying a standard concentration inequality (see \Lemma{covarEstimation}), we have $\|\Sigma^{-1/2} \wSigma \Sigma^{-1/2} - I_d \| \leq C\left( \sqrt{\frac{d+\ln(4/\beta)}{s}} + \frac{d + \ln(4/\beta)}{s} \right)$.
    Choosing $s \geq 800^2 C^2 (d + \ln(4/\beta))$ gives that $\|\Sigma^{-1/2} \wSigma \Sigma^{-1/2} - I_d \| \leq 1/400$.
    Applying \Lemma{SpectralBounds} gives that $\dist(\wSigma, \Sigma) \leq 1/100$ with probability $1-\beta/2$ as required.
\end{proof}

\begin{algorithm}
    \caption{A private algorithm for learning the covariance matrix up to a constant factor with respect to the spectral distance}
    \AlgorithmName{covariance}
    \textbf{Input:} Dataset $D = (X_1, \ldots, X_m)$; parameters $\eps, \delta, \alpha, \beta \in (0,1)$

    \begin{algorithmic}[1]
        \LineComment{{\footnotesize Some parameter settings.}}
        \State $\eta \gets \frac{1}{C_1(\sqrt{d} + \sqrt{\ln(4/\beta)})}$ \Comment{{\footnotesize $C_1$ from \Lemma{CovarianceAccuracy}}}
        \State $\gamma \gets \min\left\{
                \sqrt{\frac{\eps}{2d(d+1/\eta^2)}},
                \frac{\eps}{8d \sqrt{\ln(2/\delta)}},
                \frac{\eps}{8\ln(2/\delta)},
                \frac{\eps \eta}{12 \sqrt{d\ln(2/\delta)}}
            \right\}$.
        \State $r \gets 1$, $\phi \gets 1$
        \State $k \gets \max\left\{ \frac{400(r+\phi)}{\gamma}, \frac{20}{\eps} \ln\left( 1 + \frac{e^{\eps} - 1}{2\delta} \right) \right\}$
        \State $t \gets 3/2$

        \vspace{.3em}
        \Function{$\cA$}{$X_1, \ldots, X_s$}
        \State $\wSigma = \frac{1}{s} \sum_{\ell=1}^s X_\ell X_\ell\transpose$
        \State \textbf{Return} $\wSigma$.
        \EndFunction
        \vspace{.3em}
        \Function{$\cB$}{$\wSigma$}
            \State Let $G$ be $d \times d$ matrix with independent $\cN(0,1)$ entries.
            \State \textbf{Return} $\wSigma^{1/2} (I + \eta G)(I + \eta G)\transpose \wSigma^{1/2}$.
        \EndFunction
        \vspace{.3em}
    \end{algorithmic}
\end{algorithm}
Now we are ready to apply \Lemma{new_reduction}. The following theorem shows that we can learn the covariance matrix up to a constant distance with respect to the spectral distance using more or less $O(d^{2}/ \varepsilon)$ samples.
\begin{theorem}
    \TheoremName{covariance_spectral}
    Applying \Lemma{new_reduction} with parameters specified in \Algorithm{covariance} gives a mechanism $\cM$ satisfying the following.
    For all $m, d \in \bN$, $\eps, \delta, \beta \in (0, 1)$, and dataset $D = (X_1, \ldots, X_m)$:
    \begin{enumerate}
        \item Given $\eps, \delta, \beta, D$ as input, $\cM$ is $(2\eps, 4e^\eps \delta)$-DP and runs in time $O(md^2 + kd^3)$.
        \item If $\Sigma \succ 0$, $X_1, \ldots, X_m \stackrel{\text{i.i.d.}}{\sim} \cN(0, \Sigma)$, and $m = \wtilde{\Omega}\left( \frac{d^2 \ln(1/\beta)^{3/2} \ln(1/\delta)^{1/2}}{\eps} \right)$
            then with probability at least $1-\beta$, $\cM$ outputs a matrix $\wSigma$ such that $\dist(\Sigma, \wSigma) \leq 1/10$.
    \end{enumerate}
\end{theorem}
\begin{proof}
    \Lemma{semimetric_cov} shows that $(\cS^d, \dist)$ satisfies the assumptions of \Definition{semimetric}.
    \Lemma{CovarianceTechnical} (with $r = \phi = 1, t = 3/2$) shows that $\cB$ (as defined in \Lemma{CovarianceTechnical} and \Algorithm{covariance}) as a $(\gamma, \eps, \delta)$-masking mechanism.
    Since $k \geq 400(r+\phi)/ \gamma$ (so $400(r+\phi)/k \leq \gamma$) it follows that $\cB$ is also a $(400(r+\phi)/k, \eps, \delta)$-masking mechanism.
    Therefore, privacy immediately follows from the first assertion of \Lemma{semimetric_cov}.
    
    Next, if $m \geq C_2(d + \ln(4k/\beta)) \cdot k = \wtilde{O}\left( \frac{d(d\sqrt{\ln(1/\delta)} + \ln(1/\delta)) \ln(1/\beta)^{3/2}}{\eps} \right)$ then $\wSigma_i = \cA(\{X_\ell\}^{is}_{\ell = (i-1)s + 1})$ satisfies $\dist(\Sigma, \wSigma_i) \leq 1/10 \leq r/t$ for all $i \in [k]$ with probability $1-\beta/2$ by \Lemma{CovarConcentration} and a union bound over $k$.
    In addition, $\cB$ is $(1/100, \beta/2)$-concentrated.
    
    To conclude, let $Y_i = \wSigma_i$, $Y^* = \Sigma$, and $\what{Y} = \wSigma$ be the output of \Algorithm{reduction} with the modules specified in \Algorithm{covariance}.
    Then, taking $\alpha_1 = \alpha_2 = 1/25$ in \Lemma{semimetric_cov} gives that $\dist(\wSigma, \Sigma) \leq 1/10$ with probability $1-\beta$.
\end{proof}

Assuming $\Sigma \succ 0$, \Theorem{covariance_spectral} gives us a covariance matrix $\wSigma$ such that $0.9 I_d \preceq \wSigma^{-1/2} \Sigma \wSigma^{-1/2} \preceq 1.1 I_d$.
Note that if $X \sim \cN(0, \Sigma)$ then $\wSigma^{-1/2} X \sim \cN(0, \wSigma^{-1/2} \Sigma \wSigma^{-1/2})$.
We can now make use of off-the-shelf private covariance methods which are efficient when the covariance matrix is well-conditioned \cite{biswas2020coinpress, kamath2019privately}.
For concreteness, we use the following theorem which is paraphrased from \cite[Theorem 1.1]{kamath2019privately}, \cite[Theorem 3.3]{biswas2020coinpress}.
\begin{theorem}[\cite{biswas2020coinpress, kamath2019privately}]
    \TheoremName{WellConditionedGaussian}
    There is a polynomial time $(\eps, \delta)$-DP algorithm that given
    \[
        m = \wtilde{\Omega}\left( \left( \frac{d^2}{\alpha^2} + \frac{d^2 \sqrt{\ln(1/\delta)}}{\alpha \eps} \right) \cdot \ln(1/\beta) \right)
    \]
    samples from a Gaussian $\cN(0, \Sigma)$ with unknown covariance matrix $\Sigma$ such that $0.9 I \preceq \Sigma \preceq 1.1I$ and outputs $\wSigma$ satisfying
    $\| \wSigma^{-1/2} \Sigma \wSigma^{-1/2} - I_d \|_F \leq \alpha$. 
    In particular, $\dtv{\cN(0, \Sigma)}{ \cN(0, \wSigma)} \leq \alpha$.
\end{theorem}

Finally, we prove our result for general, full-rank covariance matrices. Note that one can use \Theorem{subspace} to easily extend this to the case where $\Sigma$ is not necessarily full-rank. The details can be found in \Subsection{complete_Gaussian_learning}.

\begin{theorem}
    \TheoremName{covariance}
    There exists a mechanism $\cM$ satisfying the following.
    For all $m, d \in \bN$, $\eps, \delta, \alpha, \beta \in (0, 1)$, and dataset $D = (X_1, \ldots, X_m)$:
    \begin{enumerate}
        \item Given $\eps, \delta, \alpha, \beta, D$ as input, $\cM$ is $(3\eps, (4e^\eps+1) \delta)$-DP and runs in polynomial time.
        \item If $\Sigma \in \psd^d$, $X_1, \ldots, X_m \stackrel{\text{i.i.d.}}{\sim} \cN(0, \Sigma)$, and
        \[
            m = \wtilde{\Omega}\left( \frac{(d^2\sqrt{\ln(1/\delta)} + d\ln(1/\delta))\ln(1/\beta)^{3/2}}{\eps} + \left( \frac{d^2}{\alpha^2} + \frac{d^2 \sqrt{\ln(1/\delta)}}{\alpha \eps} \right) \cdot \ln(1/\beta) \right)
        \]
        then with probability $1-\beta$, $\cM$ outputs a matrix $\wSigma$ such that $\dtv{\cN(0,\Sigma)}{\cN(0,\wSigma)}\leq \alpha$.
    \end{enumerate}
\end{theorem}
\begin{proof}
Combining \Theorem{covariance_spectral} and \Theorem{WellConditionedGaussian} concludes the result. The privacy follows from basic composition since the algorithm has only three steps. The utility directly follows from the adopted theorems.
\end{proof}
    
\subsection{An efficient algorithm for learning a Gaussian}
\SubsectionName{complete_Gaussian_learning}
In this subsection, we sketch the ideas needed to learn a Gaussian $\cN(\mu, \Sigma)$ under $(\eps, \delta)$-DP with no assumptions on $\mu, \Sigma$.
Suppose we draw $2m_1 + 2m_2 + m_3 + m_4$ samples from a Gaussian distribution $\cN(\mu, \Sigma)$, where $m_1$ is the sample complexity for learning a subspace using \Algorithm{subspace} and $m_2$ is the sample complexity for learning a full-rank covariance matrix using \Algorithm{covariance} and \Theorem{WellConditionedGaussian}.
The values of $m_3$ and $m_4$ are discussed below.
We now show how to privately and efficiently learn a Gaussian in four steps.
\begin{enumerate}
    \item
    Let $Z_i = \frac{X_i - X_{m_1+i}}{\sqrt{2}}$ for $i \in [m_1]$.
    We use $Z_1, \ldots, Z_{m_1}$
    as input into the private subspace learning algorithm \Algorithm{subspace}.
    As long as $m_1$ is sufficiently large, with probability $1-\beta$, it will output a matrix $P$ which is the projection onto the column space of $\Sigma$.
    \item For $i \in [m_2]$, let $Z_i = \frac{X_{2m_1 + i} - X_{2m_1+m_2+i}}{\sqrt{2}}$.
    Also, write $P = U_r U_r\transpose$ where $U_r$ is a $d \times r$ matrix whose rows form an orthonormal basis for the column space of $\Sigma$ (equivalently, $P$).
    Note that $U_r\transpose Z_i \sim \cN(0, U_r\transpose \Sigma U_r)$ and that $U_r \transpose \Sigma U_r$ is a full rank $r \times r$ covariance matrix.
    We use \Algorithm{covariance} to learn an $r \times r$ covariance matrix $\wSigma_r$ such that $\|\wSigma_r^{-1/2} U_r\transpose \Sigma U_r \wSigma_r^{-1/2} - I_r \|_F \leq \alpha$. In particular, $\dtv{\cN(0, \wSigma_r)}{\cN(0, U_r\transpose \Sigma U_r)} \leq O(\alpha)$ (see \Theorem{covariance}).
\end{enumerate}
Next, we learn the mean $\mu$ in two steps.
\begin{enumerate}[resume]
    \item For $i \in [m_3]$, let $Z_i = X_{2m_1+2m_2+i}$.
    Then $\wSigma_r^{-1/2} U_r\transpose Z_i \sim \cN(\wSigma_r^{-1/2} U_r\transpose \mu, \wSigma_r^{-1/2} U_r\transpose \Sigma U_r \wSigma_r^{-1/2})$.
    Note that (for $\alpha \leq 0.5$), we have $0.5 I \preceq \wSigma_r^{-1/2} U_r\transpose \Sigma U_r \wSigma_r^{-1/2} \preceq 1.5I$.
    In particular, each coordinate of $\wSigma_r^{-1/2} U_r \transpose Z_i$ has a variance which is bounded below by $1/2$ and above by $1.5$.
    Hence, applying \cite[Theorem 1.3]{karwa2018finite}, we can obtain an estimate $\wmu_r \in \bR^d$ such that $\|\wmu_r - \wSigma_r^{-1/2} U_r\transpose \mu\|_2 \leq \alpha$ provided $m_3 = \wtilde{\Omega}\left( \frac{d}{\alpha^2} + \frac{d \ln(1/\delta)}{\alpha \eps} \right)$ where $\wtilde{\Omega}$ hides $\polylog(1/\beta, d, \ln(1/\delta), 1/\eps, 1/\alpha)$.
    We remark that this step can be done in polynomial time.
\end{enumerate}
We now have that
\begin{align*}
    \dtv{\cN(U_r\transpose \mu, U_r\transpose \Sigma U_r)}{\cN(\wSigma_r^{1/2} \wmu_r, \wSigma_r)}
    & \leq
    \dtv{\cN(U_r\transpose \mu, U_r\transpose \Sigma U_r)}{\cN(U_r\transpose \mu, \wSigma_r)} \\
    & +
    \dtv{\cN(U_r \transpose \mu, \wSigma_r)}{\cN(\wSigma_r^{1/2} \mu_r, \wSigma_r)} \\
    &
    = \dtv{\cN(0, U_r\transpose \Sigma U_r)}{\cN(0, \wSigma_r)} \\
    & + \dtv{\cN(\wSigma_r^{-1/2} U_r\transpose \mu, I_r)}{\cN(\mu_r, I_r)} \\ 
    & \leq O(\alpha),
\end{align*}
where the first line uses the triangle inequality,
the second line uses the fact that total variation distance is invariant under bijective transformations, and the last line uses steps 2 and 3 above
as well as \Fact{KLNormal} and \Lemma{Pinsker} to assert that the total variation distance is $O(\alpha)$.
Using the fact that bijective transformations are invariant once again, that $P = U_r U_r\transpose$, and that $P \Sigma P = \Sigma$, we have
\[
    \dtv{\cN(P \mu, \Sigma)}{\cN(U_r \wSigma_r^{1/2} \wmu_r, U_r \wSigma_r U_r \transpose)} 
    =
    \dtv{\cN(U_r\transpose \mu, U_r\transpose \Sigma U_r)}{\cN(\wSigma_r^{1/2} \wmu_r, \wSigma_r)} \leq O(\alpha).
\]
All that remains is to learn $(I-P)\mu$.
\begin{enumerate}[resume]
    \item Now let $Z_i = X_{2m_1+2m_2+m_3+i}$ for $i \in [m_4]$.
    Note that $(I-P)$ is the orthogonal projection matrix onto the orthogonal complement of the column space of $\Sigma$.
    In particular $(I-P) Z_i = (I-P) \mu$ with probability $1$ if $Z_i \sim \cN(0, \Sigma)$.
    Learning $(I-P)\mu$ can now be done using standard techniques (for example, lazy private histograms or propose-test-release) with sample complexity $O(\log(1/\delta) / \eps)$.
    Let us call the output of this step $\wmu_{\perp}$.
\end{enumerate}
With the last step (and again, using that bijective maps does not affect the total variation distance to add in $\wmu_{\perp}$), we have
\[
    \dtv{\cN(\mu, \Sigma)}{\cN(U_r\wSigma_r^{1/2} \wmu_r + \wmu_{\perp}, U_r \wSigma_r U_r\transpose)} \leq O(\alpha)
\]

To conclude, we have the following theorem.
We remind the reader that $\Sigma$ is not assumed to be full-rank.
\begin{theorem}
    \TheoremName{full_gaussian}
    There exists a mechanism $\cM$ satisfying the following.
    For all $m, d \in \bN$, $\eps, \delta, \alpha, \beta \in (0, 1)$, and dataset $D = (X_1, \ldots, X_m)$:
    \begin{enumerate}
        \item Given $\eps, \delta, \alpha, \beta, D$ as input, $\cM$ is $(O(\eps), O(e^\eps\delta))$-DP and runs in polynomial time.
        \item If $X_1, \ldots, X_m \sim \cN(\mu, \Sigma)$ and
        \[
            m = \wtilde{\Omega}\left(
            \frac{d^2}{\alpha^2} + \frac{d^2 \sqrt{\ln(1/\delta)}}{\alpha \eps} + \frac{d\ln(1/\delta)}{\alpha \eps} \right) 
        \]
        then with probability $1-\beta$, $\cM$ outputs a matrix $\wSigma \in \bR^{d \times d}$ and a vector $\wmu \in \bR^d$ such that $d_{TV}(\cN(\mu,\Sigma),\cN(\wmu,\wSigma))\leq \alpha$.
        Here, the $\wtilde{\Omega}$ hides $\polylog(d, 1/\alpha, 1/\beta, 1/\eps, \ln(1/\delta))$.
    \end{enumerate}

\end{theorem}

%% file: robust_cov.tex
\section{Efficient and robust learning of the covariance matrix}
\SectionName{covariance_robust}
From the previous sections we know that the spectral $\dist$ function forms a convex semimetric space for positive definite matrices. Furthermore, we already have a masking mechanism for covariance matrices (see~\Lemma{CovarianceTechnical}) which is shown to be concentrated (see~\Lemma{CovarianceAccuracy}). In this section we show how one can replace the non-robust subroutines of \Section{covariance}--- namely \Lemma{CovarConcentration} and
\Theorem{WellConditionedGaussian}---with robust ones to get a robust and efficient learning method for Gaussians. 

Following \cite{diakonikolas2019robust, lai2016agnostic}, we use the notion of $\alpha$-corrupted sample from a distribution $\cD$, where an adversary receives an i.i.d. sample from $\cD$, alters at most an $\alpha$-fraction of them arbitrarily, and outputs the corrupted sample. The following result is based on Theorem~4.35 in \cite{diakonikolas2019robust} and is our main tool for robust learning of the covariance matrix.

\begin{theorem}[\cite{diakonikolas2019robust}]
\TheoremName{covariance_nonprivate_robust}
There are constants $C_4 > 0, 0 < \alpha_0 < 1/2$ such that the following holds.
There exist an efficient algorithm that receives $0<\alpha<\alpha_0$, $\beta>0$, and $X_1, \ldots, X_m$ as input, and outputs $\what{\Sigma}$ with the following guarantee: if $X_1, \ldots, X_m$ are $\alpha$-corrupted samples from $\cN(0,\Sigma)$ for some $\Sigma\in \pd^d$, and if $m=\wtilde{\Omega}\left(\frac{d^2\log^5 1/\beta}{\alpha^2}\right)$, then $\left\|\Sigma^{1/2}\what{\Sigma}^{-1} \Sigma^{1/2} - I\right\|_F \leq C_4\alpha \ln(1/\alpha) < 1/9$.
\end{theorem}
Our plan is use a very similar algorithm to \Algorithm{covariance}.
The key differences are that we use \Theorem{covariance_nonprivate_robust} for $\cA$ and a smaller value for $\eta$.
The proof of the following proposition is very similar to that of \Lemma{CovarianceAccuracy} and is omitted for brevity.

\begin{proposition}
    \PropositionName{CovarianceAccuracy_robust}
    Let $\Sigma \succ 0$ and $\alpha\in (0,1)$. Set $\eta = \frac{\alpha}{C_1(d + \sqrt{d\ln(4/\beta)})}$ for a sufficiently large (universal) constant $C_1 > 0$.
    Then the masking mechanism $\cB(\Sigma)$ defined in \Lemma{CovarianceTechnical} is $(\alpha/\sqrt{d}, \beta/2)$-concentrated.    
\end{proposition}

\begin{algorithm}
    \caption{A private and robust algorithm for learning the covariance matrix}
    \AlgorithmName{covariance_robust}
    \textbf{Input:} Dataset $D = (X_1, \ldots, X_m)$; parameters $\eps, \delta, \alpha, \beta \in (0,1)$

    \begin{algorithmic}[1]
        \LineComment{{\footnotesize Some parameter settings.}}
        \State $\phi \gets 1$, $t \gets 3/2$, $r \gets 1$
        \State $\eta \gets \frac{\alpha}{C_1(d + \sqrt{d\ln(4/\beta)})}$ \Comment{{\footnotesize $C_1$ from \Proposition{CovarianceAccuracy_robust}}}
        \State $\gamma \gets \min\left\{
                \sqrt{\frac{\eps}{2d(d+1/\eta^2)}},
                \frac{\eps}{8d \sqrt{\ln(2/\delta)}},
                \frac{\eps}{8\ln(2/\delta)},
                \frac{\eps \eta}{12 \sqrt{d} \sqrt{\ln(2/\delta)}}
            \right\}$. 
        \State $k \gets \max\left\{ \frac{400(r+\phi)}{\gamma}, \frac{20}{\eps} \ln\left( 1 + \frac{e^{\eps} - 1}{2\delta} \right) \right\}$

        \vspace{.3em}
        \Function{$\cA$}{$X_1, \ldots, X_s,\alpha, \beta$}
        \State $\wSigma = $ LearnCovariance$(\alpha, \beta/k, X_1, \ldots, X_s)$\Comment{{\footnotesize Algorithm from \cite{diakonikolas2019robust}}}
        \State \textbf{Return} $\wSigma$.
        \EndFunction
        \vspace{.3em}
        \Function{$\cB$}{$\wSigma$}
            \State Let $G$ be $d \times d$ matrix with independent $\cN(0,1)$ entries.
            \State \textbf{Return} $\wSigma^{1/2} (I + \eta G)(I + \eta G)\transpose \wSigma^{1/2}$.
        \EndFunction
        \vspace{.3em}

        \comment{
        \Function{LearnCovariance}{$D, \eps, \delta$}
        \State Let $s \gets \lfloor m / k \rfloor$.
        \State For $i \in [k]$, let $\wSigma_i \gets \cA(\{X_\ell\}_{\ell = (i-1)s+1}^{is})$.
        \State For $i \in [k]$, let $q_i \gets \frac{1}{k+3} |\{ j \in [k] \,:\, \dist(\wSigma_i, \wSigma_j) \leq \gamma / 4 \}|$.
        \Comment{{\footnotesize $\dist$ defined in \Equation{CovarianceDistance}.}}
        \State Let $Q \gets \frac{1}{k} \sum_{i \in [k]} q_i$.
        \State Let $Z \sim \TLap(2/k, \eps, \delta)$.
        \State Let $\what{Q} \gets Q + Z$.
        \State If $\widehat{Q} < 0.7 + \frac{2}{k\eps} \ln\left( 1 + \frac{e^{\eps} - 1}{2\delta} \right)$, fail and return $\perp$.
        \State Let $i \in \argmax_{j \in [k]} q_j$.
        \State Return $\cB(\wSigma_i)$.
        \EndFunction
        }
    \end{algorithmic}
\end{algorithm}

\begin{theorem}
    \TheoremName{covariance_spectral_robust}
    Applying \Lemma{new_reduction} with parameters specified in \Algorithm{covariance_robust} gives a mechanism $\cM$ satisfying the following.
    For all $m, d \in \bN$, $\eps, \delta, \alpha, \beta \in (0, 1)$, and dataset $D = (X_1, \ldots, X_m)$:
    \begin{enumerate}
        \item Given $\eps, \delta, \beta, \alpha, D$ as input, $\cM$ is $(2\eps, 4e^\eps \delta)$-DP and runs in $poly(m,d)$ time.
        \item If $D$ is an $\alpha$-corrupted i.i.d. sample of size $m$ from $\cN(0, \Sigma)$, for $\Sigma \succ 0$, $\alpha$ is sufficiently small, and $m = \wtilde{\Omega}\left( \frac{d^{3.5} \ln^{5.5}(1/\beta)\ln(1/\delta) }{\eps\alpha^3} \right)$,
        then with probability at least $1-\beta$, $\cM$ outputs a matrix $\wSigma$ such that $\|\wSigma^{-1/2} \Sigma \wSigma^{-1/2} - I_d\|_F = O(\alpha \ln(1/\alpha))$.
    \end{enumerate}
\end{theorem}
\begin{proof}
    The privacy claim follows from the exact same argument as in the proof of \Theorem{covariance_spectral} (still we have $400(r+\phi)/k \leq \gamma$). Also, the time complexity of the algorithm remains polynomial since the non-private robust covariance learning method runs in polynomial time.
    
    We now prove utility although the proof does not directly make use of the utility statement in \Lemma{new_reduction}.
    For $i \in [k]$, let $\wSigma_i$ be the candidate covariance matrix computed by $\cA$ with inputs $\{X_{\ell}\}_{\ell=(i-1)s+1}^{is}$.
    Provided that $s = \wtilde{\Omega} \left( \frac{d^2\log^5(k/\beta)}{\alpha^2} \right)$, we have that $\| \Sigma^{-1/2} \wSigma_i \Sigma^{-1/2} - I_d \|_F \leq C_4 \alpha \ln(1/\alpha) < 1/9$ for $i \in [k]$ with probability $1-\beta/2$.
    Note that a sufficient condition for this is that $m \geq k \cdot s = \wtilde{\Omega} \left( k \cdot \frac{d^2 \log^5(k / \beta)}{\alpha^2}  \right) = \wtilde{\Omega}\left( \frac{d^{3.5} \ln(1/\delta) \ln(1/\beta)^{5.5}}{\eps \alpha^3} \right)$.
    Next, using the trivial fact that Frobenius norm is an upper bound on spectral norm, we also have  $\| \Sigma^{-1/2} \wSigma_i \Sigma^{-1/2} - I_d \| < 1/9$.
    By \Lemma{SpectralBounds}, we also have $\|\wSigma_i^{-1/2} \Sigma^{1/2} \wSigma_i^{-1/2} - I_d \| < 4/9$.
    Thus $\dist(\Sigma, \wSigma_i) < 4/9$, where $\dist$ is as defined in \Equation{CovarianceDistance}.
    Now, using the fact that $\dist(\cdot, \cdot)$ satisfies a $3/2$-approximate triangle inequality, we have $\dist(\wSigma_i, \wSigma_j) < 2/3$ for all $i, j \in [k]$.
    In particular, since $k$ is sufficiently large, the algorithm does not fail with probability $1$ (conditioned on the event that $\dist(\wSigma_j, \wSigma_j) < 2/3$ for all $i, j \in [k]$.
    
    Now, let $\wSigma_{\mu}$ be the weighted average computed in \Line{WeightedAverage} of \Algorithm{reduction}.
    Note that if $\dist(\wSigma_i, \wSigma_j) < 2/3$ for all $i,j\in[k]$ then $\wSigma_{\mu} = \frac{1}{k} \sum_{i \in [k]} \wSigma_i$.
    It is easy to check that since $\| \Sigma^{-1/2} \wSigma_i \Sigma^{-1/2} - I_d \|_F = O(\alpha \ln(1/\alpha))$, we also have $\| \Sigma^{-1/2} \wSigma_{\mu} \Sigma^{-1/2} - I_d \|_F = O(\alpha \ln(1/\alpha))$.
    
    Next, note that by our choice of $\eta$ in \Algorithm{covariance_robust}, \Proposition{CovarianceAccuracy_robust} shows that $\cB$ is $(\alpha/\sqrt{d}, \beta/2)$-concentrated.
    In other words, if $\wSigma = \cB(\wSigma_{\mu})$ then $\|\wSigma^{-1/2} \wSigma_{\mu} \wSigma^{-1/2} - I_d \| = O(\alpha/\sqrt{d})$.
    In particular, $\|\wSigma^{-1/2} \wSigma_{\mu} \wSigma^{-1/2} - I_d \|_F = O(\alpha)$.
    
    We are now ready to complete the proof. We have that $\dtv{\cN(0, \Sigma)}{\cN(0, \wSigma_{\mu})} \leq O(\alpha \ln(1/\alpha))$ and $\dtv{\cN(0, \wSigma_{\mu})}{\cN(0, \wSigma)} \leq O(\alpha)$ by \Fact{KLNormal} and \Lemma{Pinsker}.
    By the triangle inequality, we conclude that $\dtv{\cN(0, \Sigma)}{\cN(0, \wSigma)} \leq O(\alpha \ln(1/\alpha))$.
\end{proof}

%% file: mean_robust.tex
\section{Efficient and robust mean estimation}
\SectionName{mean_robust}


As another application of our general framework, we show how to efficiently and robustly learn the mean of a Gaussian distribution with identity covariance matrix. The following proposition indicates that every normed vector space induces a convex semimetric space, which means the Euclidean space is a convex semimetric space too.

\begin{proposition}
\PropositionName{semimetric_normed}
Every normed vector space equipped with its induced metric forms a convex semimetric space. In particular, it satisfies triangle inequality ($r=\infty$, $t=1$) and $0$-locality.
\end{proposition}
\begin{proof}
The first four properties of  \Definition{semimetric} hold for norm-induced metrics by standard arguments. Locality holds because
\begin{align*}
\left\|\sum_{i\in[k]} \alpha_i Y_i -  \sum_{i \in [k]}\alpha_i' Y_i \right\| =&
\left\|\sum_{i \in [k]} (\alpha_i-\alpha_i') (Y_i -  Y_1) \right\| \leq \\
&\sum_{i \in [k]} |\alpha_i-\alpha_i'|\left\| (Y_i-Y_1) \right\| \leq \sum_{i \in [k]} |\alpha_i-\alpha_i'|\max_{i,j}\|Y_i-Y_j \|,
\end{align*}
where in the first equality we used that $\sum_{i \in [k]} \alpha_i = \sum_{i \in [k]} \alpha_i' = 1$.
\end{proof}

The Gaussian mechanism is known to be a private mechanism for functions with low $\ell_2$-sensitivity. The following theorem presents this observation in the terminology of the masking mechanisms. 

\begin{lemma}
    \LemmaName{Gaussian_masking}
    For $\mu\in \bR^d$, define $\cB(\mu)=\mu + \eta X$ where $X\sim \cN\left(0, I\right)$ and $\eta=\frac{2\gamma\ln(1.25/\delta)}{\varepsilon}$. Then $\cB$ is a $(\gamma, \varepsilon, \delta)$-masking mechanism. Moreover, for all $\beta>0$, $\cB$ is $(\eta\sqrt{2d\ln(2/\beta)},\beta/2)$ concentrated.
\end{lemma}
\begin{proof}
The privacy result is standard and is identical to that of the Gaussian mechanism; see~Theorem A.1. in~ \cite{dwork2014algorithmic} for details.
The concentration result is also standard, and  follows directly from the Hoeffding's inequality (see~\cite[Proposition 2.5]{Wainwright19}).
\end{proof}

Now we have the ingredients to reduce the problem of privately learning the mean to the non-private one. The following result is about the polynomial learnability of the mean in the robust (non-private) setting and is a direct consequence of Theorem~1.1 in \cite{dong2019quantum}.
\begin{theorem}[\cite{dong2019quantum}]
\TheoremName{robust_mean_non_private}
There an absolute constant $C_3$ and an algorithm that receives $0<\alpha<1/2$, $\beta>0$, and $X_1, \ldots, X_m$ as input, runs in $poly(m, d)$, and outputs $\what{\mu}$ with the following guarantee: if $X_1, \ldots, X_m$ are $\alpha$-corrupted samples from $\cN(\mu,\Sigma)$ for some $(1/2)I_d \preceq \Sigma \preceq 2I_d$, and if $m={\Omega}\left(\frac{d + \log 1/\beta}{\alpha^2}\right)$, then $\left\|\mu-\what{\mu}\right\|_2\leq C_3\alpha \sqrt{\log(1/\alpha)}$.
\end{theorem}

Finally, we are ready to apply the \Lemma{new_reduction}.  \Algorithm{mean_robust} summarizes the parameters that we use for this reduction.

\begin{algorithm}
    \caption{A private and robust algorithm for learning the mean}
    \AlgorithmName{mean_robust}
    \textbf{Input:} Dataset $D = (X_1, \ldots, X_m)$; parameters $\eps, \delta, \alpha, \beta \in (0,1)$

    \begin{algorithmic}[1]
        \LineComment{{\footnotesize Some parameter settings.}}
        \State $t \gets 1, \phi \gets 0$  \Comment{{\footnotesize See \Proposition{semimetric_normed}.}}
        \State $r \gets C_3\alpha \log(1/\alpha)$
        \Comment{{\footnotesize $C_3$ is from \Theorem{robust_mean_non_private}.}}
        \State $\eta \gets \frac{\alpha}{\sqrt{2d\ln(2/\beta)}}$ 
        \State $\gamma \gets \frac{\eta \varepsilon}{2\ln(1.25/\delta)}$ 
        \State $k \gets \max\left\{ \frac{400r}{\gamma}, \frac{20}{\eps} \ln\left( 1 + \frac{e^{\eps} - 1}{2\delta} \right) \right\}$

        \vspace{.3em}
        \Function{$\cA$}{$X_1, \ldots, X_s,\alpha, \beta$}
        \State $\what{\mu} = $ QUEScoreFilter$(\alpha, \beta/k, X_1, \ldots, X_s)$\Comment{{\footnotesize Algorithm from \cite{dong2019quantum}}}
        \State \textbf{Return} $\what{\mu}$.
        \EndFunction
        \vspace{.3em}
        \Function{$\cB$}{$\what{\mu}$}
            \State Let $X\sim \cN(0, I_d)$
            \State \textbf{Return} $\what{\mu} + \eta X$.
        \EndFunction
        \vspace{.3em}

        \comment{
        \Function{LearnMean}{$D, \eps, \delta$}
        \State Let $s \gets \lfloor m / k \rfloor$.
        \State For $i \in [k]$, let $\wSigma_i \gets \cA(\{X_\ell\}_{\ell = (i-1)s+1}^{is})$.
        \State For $i \in [k]$, let $q_i \gets \frac{1}{k+3} |\{ j \in [k] \,:\, \dist(\wSigma_i, \wSigma_j) \leq \gamma / 4 \}|$.
        \Comment{{\footnotesize $\dist$ defined in \Equation{CovarianceDistance}.}}
        \State Let $Q \gets \frac{1}{k} \sum_{i \in [k]} q_i$.
        \State Let $Z \sim \TLap(2/k, \eps, \delta)$.
        \State Let $\what{Q} \gets Q + Z$.
        \State If $\widehat{Q} < 0.7 + \frac{2}{k\eps} \ln\left( 1 + \frac{e^{\eps} - 1}{2\delta} \right)$, fail and return $\perp$.
        \State Let $i \in \argmax_{j \in [k]} q_j$.
        \State Return $\cB(\wSigma_i)$.
        \EndFunction
        }
    \end{algorithmic}
\end{algorithm}

\begin{theorem}
    \TheoremName{mean_robust}
    Applying \Lemma{new_reduction} with parameters specified in \Algorithm{mean_robust} gives a mechanism $\cM$ satisfying the following.
    For all $m, d \in \bN$, $\eps, \delta, \alpha, \beta \in (0, 1)$, and dataset $D = (X_1, \ldots, X_m)$:
    \begin{enumerate}
        \item Given $\eps, \delta, \beta, \alpha, D$ as input, $\cM$ is $(2\eps, 4e^\eps \delta)$-DP and runs in $poly(m,d)$ time.
        \item If $D$ is an $\alpha$-corrupted i.i.d. sample of size $m$ from $\cN(\mu, \Sigma)$, for $(1/2) I \preceq \Sigma \preceq 2I$ , $\alpha\in (0,0.5)$, and $m = \wtilde{\Omega}\left( \frac{d^{3/2} \ln(1/\delta) }{\eps\alpha^2} \right)$,
            then with probability at least $1-\beta$, $\cM$ outputs a vector $\wtilde{\mu}$ such that $\|\wtilde{\mu}-\mu\|_2 =O(\alpha \sqrt{\ln(1/\alpha)})$.
    \end{enumerate}
\end{theorem}

\begin{proof}
    \Proposition{semimetric_normed} indicates that $(\bR^d, \|.\|_2)$ form a convex semimetric space with non-restricted triangle inequality and $0$-locality. We  take $r=C_3\alpha\sqrt{\log(1/\alpha)}<\infty$.
    By \Lemma{Gaussian_masking}, $\cB$ is $(\gamma, \eps, \delta)$-masking mechanism with $\gamma, \eta$ as defined in \Algorithm{mean_robust}.
    Our choice of $k$ in \Algorithm{mean_robust} means that $400r / k \leq \gamma$ so $\cB$ is also a $(400r/k, \eps, \delta)$-masking mechanism.
    Therefore, the privacy of the algorithm follows from \Lemma{new_reduction}. Also, the time complexity of the algorithm remains polynomial since the non-private robust mean estimation runs in polynomial time.
    
    \Lemma{Gaussian_masking} shows that $\cB$ is $(\eta\sqrt{2d\ln(2/\beta)}, \beta/2)$-concentrated. By our choice of $\eta$ in \Algorithm{mean_robust}, we conclude that $\cB$ is $(\alpha, \beta/2)$-concentrated.
    
    Let $\mu_i=\cA(\{X_\ell\}^{is}_{\ell = (i-1)s + 1}, \alpha, \beta)$ be the outputs of the non-private mean estimation method. Since $m = \wtilde{\Omega} \left(k \cdot \frac{d+\log(k/\beta)}{\alpha^2} \right)$, we conclude that $\|\mu_i - \mu\|_2=O(\alpha \sqrt{\log(1/\alpha)})$ for all $i \in [k]$ with probability $1-\beta/2$ by \Theorem{robust_mean_non_private} and a union bound over $k$.

    Finally, setting $Y_i = \mu_i$, $Y^* = \mu$, and $\what{Y} = \wtilde{\mu}$ for \Algorithm{reduction},  with the modules specified in \Algorithm{covariance_robust}.
    Then, taking $\alpha_1 = \alpha_2 = O(\alpha\sqrt{\log(1/\alpha)})$ in \Lemma{new_reduction} gives that $\dist(\wtilde{\mu}, \mu) =O(\alpha\sqrt{\log(1/\alpha)})$ with probability $1-\beta$. Moreover, it is enough for the number of samples to be $m=\wtilde{\Omega} \left(k\cdot \frac{d+\log(k/\beta)}{\alpha^2} \right)
    =\wtilde{\Omega}\left( \frac{d^{3/2} \ln(1/\delta) }{\eps\alpha^2} \right)$, which concludes the proof. 
\end{proof}

Finally, we can combine the mean estimation approach of \Theorem{mean_robust} with the covariance estimation of \Theorem{covariance_spectral_robust} to have a complete private algorithm for robust learning of Gaussians. The sample complexity bottleneck will be that of covariance estimation. The proof of the following corollary is very similar to that of \Subsection{complete_Gaussian_learning} and is omitted for brevity. 

\chris{Check, below should probably be $d^{3.5}$.}
\begin{corollary}
\CorollaryName{robust_gaussian}
There exist an $(O(\varepsilon), O(e^\eps\delta))$-DP algorithm that receives $0<\alpha<1/2$, $\beta>0$, and $X_1, \ldots, X_m$ as input, and outputs $\what{\Sigma}, \what{\mu}$ with the following guarantee: if $X_1, \ldots, X_m$ are $\alpha$-corrupted samples from $\cN(\mu,\Sigma)$ for $\Sigma\in \pd^d,\mu \in \bR^d$ then $d_{TV}(\cN(\mu, \Sigma),\cN(\what{\mu}, \what{\Sigma}))=O(\alpha \log(1/\alpha))$ as long as $m={\Omega}\left(\frac{d^{4.5}\ln(1/\delta)\ln^5(d/\varepsilon\alpha\beta)}{\varepsilon\alpha^3}\right)$.
\end{corollary}

\begin{remark}
Unlike \Theorem{full_gaussian} for the non-robust case, here the utility guarantee is conditioned on the assumption that $\Sigma \succ 0$. In order to address this, one needs to come up with either a variant of the robust covariance estimation approach of \cite{diakonikolas2019robust} which works for singular covariance matrices, or an efficient and robust subspace recovery method which can be a hard task \cite{hardt2013algorithms}.
\end{remark}

\begin{remark}
The sample complexities of \Theorem{covariance_spectral_robust} (and therefore \Corollary{robust_gaussian}) as well as \Theorem{mean_robust} are not optimal. It seems possible to improve these sample complexities by a more careful analysis. We leave this as a future research direction.
\end{remark}

%% file: app_facts.tex
\section{Standard Facts}
\AppendixName{StandardFacts}
\begin{claim}
    \ClaimName{ratio-ineq}
    If $x \leq 1/2$ then $\frac{1+x}{1-x} \leq 1+4x$.
\end{claim}
\begin{proof}
    Let $f(x) = 1+4x - \frac{1+x}{1-x}$.
    Then $f''(x) = -\frac{2(1+x)}{(1-x)^3} - \frac{2}{(1-x)^2}$.
    So $f$ is concave on the interval $[0,1)$.
    Since $f(0) = f(1/2) = 0$, we have that $f(x) \geq 0$ for $x \in [0, 1/2]$ as desired.
\end{proof}
\begin{claim}
    \ClaimName{ratio-ineq2}
    For all $x \in \bR$, $\frac{1-x}{1+x} \geq 1-2x$.
\end{claim}
\begin{proof}
    We have $\frac{1-x}{1+x} \geq (1-x)^2 = 1-2x+x^2 \geq 1-2x$.
\end{proof}

\begin{lemma}[\protect{\cite[Lemma 1]{LM00}}]
    \LemmaName{ChiSquareBound}
    Let $g_1, \ldots, g_k$ be i.i.d.~$\cN(0, 1)$ random variables.
    Then
    \[
        \prob{\sum_{i=1}^k g_i^2 - k \geq 2 \sqrt{k t} + 2t } \leq e^{-t}.
    \]
\end{lemma}

\begin{lemma}[\protect{\cite[Exercise~4.7.3]{Ver18}}]
    \LemmaName{covarEstimation}
    There is an absolute constant $C > 0$ such that the following holds.
    Let $\Sigma \succ 0$ and $\beta > 0$.
    Let $X_1, \ldots, X_m \sim \cN(0, \Sigma)$ and $\what{\Sigma} = \frac{1}{m} \sum_{i=1}^m X_i X_i\transpose$.
    Then $\|\Sigma^{-1/2} \wSigma \Sigma^{-1/2} - I_d \| \leq C\left( \sqrt{\frac{d + \ln(2/\beta)}{m}} + \frac{d + \ln(2/\beta)}{m} \right)$ with probability $1-\beta$.
\end{lemma}

\begin{theorem}[\protect{\cite[Theorem 4.4.5]{Ver18}}]
    \TheoremName{GaussianSpectralNorm}
    There is a universal constant $C > 0$ such that the following holds.
    Let $G$ be a $d\times d$ random matrix where each entry is an independent $\cN(0, 1)$ random variable.
    Then $\|G\| \leq C(\sqrt{d} + \sqrt{\ln(2/\beta)})$ with probability $1-\beta$.
\end{theorem}

\begin{lemma}
    \LemmaName{SpectralBounds}
    Let $A, B$ be $d \times d$ matrices and suppose that $\|A^{-1/2} B A^{-1/2} - I_d\| \leq \gamma \leq 1/2$.
    Then $\|B^{-1/2} A B^{-1/2} - I_d\| \leq 4 \gamma$.
\end{lemma}
\begin{proof}
    The assumption that $\|A^{-1/2} B A^{-1/2} - I_d \| \leq \gamma$ means that $(1-\gamma)A \preceq B \preceq (1+\gamma) A$.
    Hence, we also have that $\frac{1}{1+\gamma} B \preceq A \preceq \frac{1}{1-\gamma} B$.
    The asserted inequality is then proved using the bounds in \Claim{RecipLinearBound}.
\end{proof}
\begin{claim}
    \ClaimName{RecipLinearBound}
    Suppose that $\gamma \in [0,1/2]$.
    Then $\frac{1}{1+\gamma} \geq 1-4\gamma$ and $\frac{1}{1-\gamma} \leq 1+4\gamma$.
\end{claim}
\begin{proof}
    To prove the first inequality, let $f(\gamma) = 1-4\gamma - \frac{1}{1+\gamma}$.
    We need to show that $f(\gamma) \geq 0$ for all $\gamma \in [0, 1/2]$.
    Differentiating, we have $f'(\gamma) = 4 - \frac{1}{1+x^2}$ so $f$ is increasing.
    Since $f(0) = 0$, the first inequality follows.
    
    For the second inequality, let $f(\gamma) = 1+4\gamma - \frac{1}{1-\gamma}$.
    The derivative $f'(\gamma) = 4 - \frac{1}{(1-\gamma)^2} \geq 0$ for $\gamma \in [0, 1/2]$.
    So $f$ is increasing on $[0, 1/2]$ and $f(0) = 0$ so the second inequality is proved.
\end{proof}

%% file: app_prelim.tex
\section{Missing Proofs From \Section{prelim}}
\AppendixName{prelim}
\subsection{Proof of \Fact{KLNormal}}
\AppendixName{KLNormal}
\begin{proof}[Proof of \Fact{KLNormal}]
    The first assertion is a well-known formula for the KL divergence between two normal distributions (see \cite[\S 9.1]{Kullback68} or \cite[Eq.~A.23]{RW06}).
    
    Next, let $\lambda_1, \ldots, \lambda_d$ be the eigenvalues of $\Sigma_2^{-1} \Sigma_1$.
    Then
    \begin{align*}
        \Tr(\Sigma_2^{-1} \Sigma_1 - I) - \ln \det(\Sigma_2^{-1} \Sigma_1)
        = \sum_{i=1}^d \lambda_i - 1 - \ln \lambda_i
        \leq \sum_{i=1}^d (\lambda_i - 1)^2
        = \| \Sigma_2^{-1} \Sigma_1 - I \|_F^2,
    \end{align*}
    where the inequality follows from \Claim{KLInequality}.
    The lemma now follows because the two matrices $\Sigma_2^{-1} \Sigma_1$ and $\Sigma_2^{-1/2} \Sigma_1 \Sigma_2^{-1/2}$ have the same spectrum.
\end{proof}
\begin{claim}
    \ClaimName{KLInequality}
    Suppose that $x \geq 1/2$. Then $x - 1 - \ln(x) \leq (x-1)^2$.
\end{claim}
\begin{proof}
    Let $f(x) = (x-1)^2 - (x-1) + \ln(x)$.
    Then $f'(x) = 2(x-1) - 1 + \frac{1}{x} = \frac{2x^2 - 3x + 1}{x} = \frac{(2x-1)(x-1)}{x}$.
    So $f'(x) \leq 0$ on $[1/2, 1]$ and $f'(x) \geq 0$ on $[1, \infty)$.
    Hence, on the interval $[1/2, \infty)$, $f(x)$ is minimized at $x = 1$ and $f(x) \geq f(1) = 0$ for all $x \in [1/2, \infty)$.
\end{proof}
\begin{claim}
    \ClaimName{sqrtIneq}
    Suppose $x \leq 1/2$.
    Then $\sqrt{1-x} \geq 1-x/2 - x^2/2$.
\end{claim}
\begin{proof}
    Let $f(x) = \sqrt{1-x} - 1 + x/2 + x^2 / 2$.
    Then $f'(x) = x - \frac{1}{2\sqrt{1-x}} + \frac{1}{2}$ and $f''(x) = 1 - \frac{1}{4(1-x)^{3/2}}$.
    Note that $f'(0) = 0$ so $x = 0$ is a critical point of $f$.
    Moreover, $f''(x) \geq 0$ for all $x \leq 1/2$ so $f$ is convex on $(-\infty, 1/2)$.
    Hence, the minimizer of $f$ on the interval $(-\infty, 1/2]$ is at $x = 0$ and $f(x) \geq f(0) = 0$ for all $x \in (-\infty, 1/2]$.
\end{proof}

\subsection{Proof of \Lemma{chiSquareConcentration}}
\AppendixName{chiSquareConcentration}
\begin{claim}[see \protect{\cite[Example~2.8]{Wainwright19}}]
    \ClaimName{chiMgf}
    Let $X \sim \cN(0,1)$ and $|u| \leq 1/4$.
    Then $\expect{e^{u (X^2 - 1)}} \leq e^{2u^2}$.
\end{claim}
We first prove \Lemma{chiSquareConcentration} for the special case where the matrix is diagonal.
The general case follows easily from rotation invariance.
\begin{lemma}
    \LemmaName{chiSquareConcentration2}
    Let $\Lambda$ be a $d \times d$ diagonal matrix and $g \sim \cN(0, I_d)$.
    Then for all $x > 0$,
    \[
        \prob{g\transpose \Lambda g \geq \Tr(\Lambda) + 4\|\Lambda\|_F \cdot \sqrt{x} + 4 \|\Lambda\| \cdot x} \leq e^{-x}.
    \]
\end{lemma}
\begin{proof}
    Note that $g\transpose \Lambda g = \sum_{i=1}^d \lambda_i (g_i^2 - 1)$ where $\lambda_1, \ldots, \lambda_d$ are the diagonal entries of $\Lambda$.
    From \Claim{chiMgf}, we have $\expect{e^{u \lambda_i(g_i^2 - 1)}} \leq e^{2u^2 \lambda_i^2}$ provided $|u| \leq 1 / 4 \|\Lambda\|$.
    Hence, $\expect{e^{u(g\transpose \Lambda g - \Tr(\Lambda))}} = \prod_{i=1}^d \expect{e^{u\lambda_i (g_i^2 - 1)}} \leq e^{2u^2 \|\Lambda\|_F^2}$.
    Applying Markov's Inequality, we thus have, for $0 \leq u \leq 1 / 4\|\Lambda\|$,
    \begin{align*}
        \prob{ g\transpose \Lambda g - \Tr(\Lambda) \geq 4 \|\Lambda\|_F \cdot \sqrt{x} + 4 \|\Lambda\| \cdot x}
        & \leq e^{-u(4 \|\Lambda\|_F \cdot \sqrt{x} + 4 \|\Lambda\| \cdot x) + 2u^2 \|\Lambda\|_F^2}.
    \end{align*}
    It is a straightforward calculation to verify that choosing $u = \min\left\{ \frac{1}{4\|\Lambda\|}, \frac{\sqrt{x}}{\|\Lambda\|_F} \right\}$ makes the RHS at most $e^{-x}$ which proves the lemma.
\end{proof}
\begin{proof}[Proof of \Lemma{chiSquareConcentration}]
    Let $A = U \Lambda U\transpose$ be the eigendecomposition of $A$.
    Then $g\transpose A g = g\transpose U \Lambda U\transpose g \eqdist g\transpose \Lambda g$ where the last equality in distribution is by rotation invariance (\Fact{RotationInvariance}).
    The claim now follows from \Lemma{chiSquareConcentration2} since $A$ and $\Lambda$ have the same spectrum.
\end{proof}
\begin{fact}
    \FactName{RotationInvariance}
    Let $g \sim \cN(0, I_d)$ and $U$ be a $d \times d$ orthogonal matrix.
    Then $Ug \sim \cN(0, I_d)$.
\end{fact}

\subsection{Proof of \Lemma{privLossDP}}
\AppendixName{privLossDP}
\begin{proof}[Proof of \Lemma{privLossDP}]
    Let $f_1, f_2$ be the densities of $\cD_1, \cD_2$, respectively.
    Let $S \subseteq \bR^d$ and let $T = \{ y \in \bR^d \,:\, f_1(y) \geq e^{\eps} f_2(y) \}$.
    The assumption $\probs{Y \sim \cD_1}{\privLoss{\cD_1}{\cD_2}(Y) \geq \eps} \leq \delta$
    means that $\probs{Y \sim \cD_1}{Y \in T} \leq \delta$.
    Thus, we have
    \begin{align*}
        \probs{Y \sim \cD_1}{Y \in S}
        & = \int_S f_1(y)\, \dd y \\
        & =  \int_{S \setminus T} f_1(y) \, \dd y + \int_{S \cap T} f_1(y) \, \dd y \\
        & \leq \int_{S \setminus T} e^{\eps} f_2(y) \, \dd y + \delta \\
        & \leq e^{\eps} \probs{Y \sim \cD_2}{S} + \delta.
    \end{align*}
    An identical argument shows that $\probs{Y \sim \cD_2}{Y \in S} \leq e^{\eps} \probs{Y \sim \cD_1}{S} + \delta$,
    which proves that $\cD_1, \cD_2$ are $(\eps, \delta)$-indistinguishable.
\end{proof}

%% file: app_reduction.tex
\section{Missing proofs from \Section{reduction}}
\AppendixName{reduction}

\subsection{Proof of \Lemma{semimetric_cov}}
\AppendixName{semimetric_cov}

\begin{proof}
    The first two conditions in \Definition{semimetric} are trivial.
    
    \textbf{Triangle Inequality.}
    
    Next, we prove that $\dist$ satisfies a $(3/2)$-approximate $1$-restricted triangle inequality.
    Let $A, B, C \in \cS^d$ and let $\rho_{ab} = \dist(A, B)$, $\rho_{bc} = \dist(B, C)$ and $\rho_{ac} = \dist(A, C)$.
    We need to show that if $\rho_{ab}, \rho_{bc} \leq 1$ then $\rho_{ac} \leq \frac{3}{2} (\rho_{ab} + \rho_{bc})$.
    Indeed, we have
    \begin{align*}
        A
        & \preceq (1+\rho_{ab}) B
        \preceq (1+\rho_{ab})(1+\rho_{bc}) C
        = (1 + \rho_{ab} + \rho_{bc} + \rho_{ab}\rho_{bc})C \\
        & \preceq (1 + \rho_{ab} + \rho_{bc} + \rho_{ab}^2/2 + \rho_{bc}^2/2) C
        \preceq (1 + \frac{3}{2} (\rho_{ab} + \rho_{bc}) ) C,
    \end{align*}
    where the first two inequalities used that $\dist(A,B) = \rho_{ab}$ and $\dist(B, C) = \rho_{bc}$, respectively, the third inequality used that $xy \leq x^2/2 + y^2 / 2$ for any $x, y \in \bR$ and the last inequality used that $\rho_{ab}, \rho_{bc} \leq 1$.
    Similarly, we can show that
    \[
        A \succeq (1 - \frac{3}{2} (\rho_{ab} + \rho_{bc})) C,
    \]
    which imply, with the previous quality, that $\|C^{-1/2} A C^{-1/2} - I \| \leq \frac{3}{2} (\rho_{ab} + \rho_{bc})$.
    We can similarly show that $\|A^{-1/2} C A^{-1/2} - I \| \leq \frac{3}{2} (\rho_{ab} + \rho_{bc})$ so $\rho_{ac} \leq \frac{3}{2}(\rho_{ab} + \rho_{bc})$, as desired.

    \textbf{Convexity.}
    
    For this, it would suffice to show that for all $\alpha \in [0,1]$,
    $\dist(\alpha A + (1-\alpha) B, C) \leq \alpha \dist(A, C) + (1-\alpha) \dist(B, C)$, and the rest follows from induction over $k$.
    Let $\rho_{a} = \dist(A, C), \rho_{b} = \dist(B, C)$.
    We need to show the following four inequalities:
    \begin{enumerate}
        \item $C \preceq (1+\alpha \rho_a + (1-\alpha) \rho_b)(\alpha A + (1-\alpha) B)$;
        \item $C \succeq (1-\alpha \rho_a - (1-\alpha) \rho_b)(\alpha A + (1-\alpha) B)$;
        \item $\alpha A + (1-\alpha) B \preceq (1+\alpha \rho_a + (1-\alpha) \rho_b) C$; and
        \item $\alpha A + (1-\alpha) B \succeq (1-\alpha \rho_a - (1-\alpha) \rho_b) C$.
    \end{enumerate}
    The third and fourth items are easy to check.
    Indeed, for the third item, we have $A \preceq (1+\rho_a) C$ and $B \preceq (1+\rho_b) C$ so
    \[
        \alpha A + (1-\alpha)B \preceq \alpha(1+\rho_a) C + (1-\alpha)(1 + \rho_b) C
        = (1 +\alpha \rho_a + (1-\alpha)\rho_b) C.
    \]
    The fourth item follows similarly.
    
    It remains to prove the first two inequalities.
    Here, we prove only the first inequality since the second follows by an analogous argument.
    Let
    \[
        K = \left(
        \frac{\alpha(1+\alpha \rho_a + (1-\alpha) \rho_b)}{1+\rho_a}
        +
        \frac{(1-\alpha)(1+\alpha \rho_a +(1-\alpha) \rho_b)}{1 + \rho_b} 
        \right)^{-1}
    \]
    and set
    \[
        w = K \frac{\alpha(1+\alpha \rho_a + (1-\alpha) \rho_b)}{1+\rho_a}.
    \]
    Our choice of $K$ means that
    \[
        1-w = K \frac{(1-\alpha)(1+\alpha \rho_a + (1-\alpha) \rho_b)}{1+\rho_b}.
    \]
    Note that $w \in [0,1]$.
    Since $C \preceq (1+\rho_a) A$ and $C \preceq (1+\rho_b) B$, we have
    \begin{align*}
        C
        & = wC + (1-w)C \\
        & \preceq w(1+\rho_a) C + (1-w) (1+\rho_b) B \\
        & = K \left[ \alpha(1+\alpha \rho_a + (1-\alpha) \rho_b) A + (1-\alpha)(1+\alpha \rho_a + (1-\alpha) \rho_b) B \right] \\
        & = K \left[ (1+\alpha \rho_a + (1-\alpha) \rho_b)(\alpha A + (1-\alpha) B) \right].
    \end{align*}
    \Claim{spectralNormConvexityIneq} shows that $K \leq 1$ so we have the desired inequality.

\begin{claim}
    \ClaimName{spectralNormConvexityIneq}
    Let $\alpha \in [0,1]$ and let $\rho_a, \rho_b \in \bR_{\geq 0}$.
    Then
    \[
        \frac{\alpha(1+\alpha \rho_a + (1-\alpha) \rho_b)}{1+\rho_a} + \frac{(1-\alpha)(1+\alpha \rho_a + (1-\alpha)\rho_b)}{1+\rho_b} \geq 1. 
    \]
\end{claim}
\begin{proof}
    Let
    \[
        f(\alpha) \coloneqq \frac{\alpha(1+\alpha \rho_a + (1-\alpha) \rho_b)}{1+\rho_a} + \frac{(1-\alpha)(1+\alpha \rho_a + (1-\alpha)\rho_b)}{1+\rho_b}.
    \]
    Observe that $f(0) = f(1) = 1$.
    Thus, it suffices to prove that $f$ is concave.
    Indeed, we have
    \[
        f'(\alpha) = \frac{1+\alpha \rho_a + (1-\alpha) \rho_b}{1+\rho_a} + \frac{\alpha(\rho_a - \rho_b)}{1+\rho_a}
        - \frac{1+\alpha \rho_a + (1-\alpha)\rho_b}{1+\rho_b} + \frac{(1-\alpha)(\rho_a - \rho_b)}{1+\rho_b}.
    \]
    and
    \[
        f''(\alpha) = 2\frac{\rho_a - \rho_b}{1+\rho_a} + 2 \frac{\rho_b - \rho_a}{1+\rho_b}
        = 2(\rho_a - \rho_b) \cdot \left( \frac{1}{1+\rho_a} - \frac{1}{1+\rho_b} \right).
    \]
    Finally, we check that $f'' \leq 0$.
    Indeed, if $\rho_a = \rho_b$ then $f'' = 0$ and if $\rho_1 \neq \rho_b$ then 
    $\rho_a - \rho_b$ and $1/(1+\rho_a) - 1/(1+\rho_b)$ have opposite signs so $f'' < 0$.
\end{proof}
\textbf{Locality.}

Let $\wSigma = \sum_{i \in [k]} w_i \wSigma_i$ and $\wSigma' = \sum_{i \in [k]} w_i' \wSigma_i$.
Let $\rho = \max_{i,j \in [k]} \dist(\wSigma_i, \wSigma_j)$.
Then
\begin{align*}
    \sum_{i \in [k]} w_i \wSigma_i & \preceq \sum_{i \in [k]} w_i' \wSigma_i + \sum_{i \in [k]} |w_i - w_i'| \wSigma_i \\
    & \preceq \wSigma' + \sum_{i \in [k]} |w_i - w_i'| (1+\rho) \wSigma' \\
    & = \left( 1 + (1+\rho) \sum_{i \in [k]} |w_i - w_i'| \right) \wSigma'
\end{align*}
For the second inequality, we used $\wSigma_i \preceq (1+\rho) \wSigma_j$ for all $i, j \in [k]$, which imply $\wSigma_i \preceq (1+\rho) \wSigma'$ for all $i \in [k]$.
\end{proof}

\subsection{Proof of \Lemma{new_reduction}}
\AppendixName{reduction_proof}

\begin{proof}
We first prove the utility and then the privacy result.
\paragraph{Utility.}
We first show that the algorithm will not return $\perp$ if the assumptions of the utility guarantee hold.
Based on the assumption, $\dist(Y_i, Y^*) \leq \alpha_1/t$ for all $i\in [k]$.
By the $t$-approximate $r$-restricted triangle inequality we have $\dist(Y_i, Y_j) \leq \alpha_1\leq r/t$ for all $i,j\in [k]$.
Therefore, for all $i\in [k]$, $q_i=1$.
Finally, $Q = 1 > 0.8+\frac{2}{k\varepsilon}\ln{\left(1+\frac{e^\varepsilon-1}{2\delta}\right)}$
and therefore the algorithm does not fail in \Line{ReductionTLap}.

Let $W = \sum_{i \in [k]}{w_i}$.
To show that $\what{Y}$ concentrates around $Y^*$ we have
\begin{align*}
&\prob{\dist(\hat{Y}, Y^*) > \alpha_1 + \alpha_2} \leq
\prob{\dist(\hat{Y}, \mu) + \dist(\mu, Y^*) > \alpha_1/t + \alpha_2/t} \leq \\
&
\prob{\dist(\hat{Y}, \mu) + \sum_{i \in [k]} \frac{w_i}{W} \dist(Y_i, Y^*) > \alpha_1/t + \alpha_2/t} \leq \\
&\prob{\dist(\hat{Y}, \mu) + \alpha_1/t  > \alpha_1/t + \alpha_2/t} \leq
\prob{\dist(\hat{Y}, \mu) > \alpha_2/t} \leq \beta
\end{align*}

where the first inequality follows from the $t$-approximate triangle inequality, and the second one follows from the convexity of the $\dist$ function and the fact that $\mu$ is a convex combination of $Y_i$'s.
The third inequality uses the assumption that $\dist(Y_i, Y^*) \leq \alpha_1 / t$ for all $i \in [k]$.
Finally, the last inequality follows from the assumption that the masking mechanism $\cB$ is $(\alpha_2/t, \beta)$ concentrated. 

\paragraph{Privacy.}

Let $D, D'$ be two neighbouring datasets and let $\cM$ denote the mechanism specified by \Algorithm{reduction}.
Note that the statistic $Q$ computed in \Line{ReductionQ} has sensitivity less than $2/k$.
This is because if we modify a single datapoint, there is at most a single $i$ such that $Y_i$ changes.
Hence, the value of $q_i$ changes by at most $1$ and for $j \neq i$, the value of $q_j$ changes by less than $1/k$.
Thus, $Q$ has sensitivity at most $2/k$.
In particular, since we apply the Truncated Laplace mechanism in \Line{ReductionTLap} (see \Theorem{TLap}), we have
\begin{equation}
    \EquationName{ReductionTLap}
    \prob{\cM(D) = \perp} \leq e^{\eps} \prob{\cM(D') = \perp} + \delta
\end{equation}

We now show that for any $T \subseteq \cY$, we have
\begin{align}
    \prob{\cM(D) \in T} & \leq e^{2\eps} \prob{\cM(D') \in T} + 4e^{\eps} \delta \quad \text{and} \EquationName{ReductionDP1} \\
    \prob{\cM(D) \in T \cup \{\perp\}} & \leq e^{2\eps} \prob{\cM(D') \in T \cup \{\perp\}} + 4e^{\eps} \delta \EquationName{ReductionDP2}
\end{align}
which establishes that \Algorithm{reduction} is $(\eps, \delta)$-DP.
To this end, we consider two different cases.

\paragraph{Case 1: $Q < 0.8$.}
In this case, $\what{Q} < 0.8 + \frac{2}{k\eps} \ln\left( 1 + \frac{e^{\eps} - 1}{2\delta} \right)$ with probability $1$ so
$\prob{\cM(D) = \perp} = 1$.
Now, we verify that \Equation{ReductionDP1} and \Equation{ReductionDP2} hold.
For any $T \subseteq \cY$, we have $\prob{\cM(D) \in T} = 0$ so \Equation{ReductionDP1} is trivially satisfied.
To check \Equation{ReductionDP2} holds, we apply \Equation{ReductionTLap} to see that
\[
    \prob{\cM(D) \in T \cup \{\perp\}}
    = \prob{\cM(D) = \perp}
    \leq e^{\eps} \prob{\cM(D') = \perp} + \delta
    \leq e^{\eps} \prob{\cM(D') \in T \cup \{\perp\}} + \delta.
\]

\paragraph{Case 2: $Q \geq 0.8$.}
Let $Y_1, \ldots, Y_k$ (resp.~$Y_1', \ldots, Y_k'$) be the output of the non-private algorithm $\cA$ in
\Line{ReductionNonPrivate} when the dataset is $D$ (resp.~$D'$).
Further, define let $q_i, w_i, \mu$ be as defined in \Algorithm{reduction} when the input is $D$ and $q_i', w_i',\mu'$ be the analogous quantities when the input is $D'$.
For notation, let $W = \sum_{i=1}^k w_i$ and $W' = \sum_{i=1}^k W'$.
Our goal is to show that $\dist(\mu, \mu') \leq 400(r+\phi)/k$ so that
the $(400(r+\phi)/k, \eps, \delta)$-masking mechanism guarantees the output is $(\eps, \delta)$-indistinguishable.

We proceed by proving a series of claims.
The first claim gives a lower bound on $W$.
\begin{claim}
    \ClaimName{W_LB}
    $W \geq k/3$.
\end{claim}
\begin{proof}
    Let $\alpha = |\{i \in [k]\,:\, q_i \geq 0.7\}| / k$.
    Then $0.8 \leq Q \leq \alpha + 0.7(1-\alpha)$ so $\alpha \geq 1/3$.
    Finally, $W \geq \alpha k$ because if $q_i \geq 0.7$ then $w_i = 1$.
    We conclude that $W \geq k/3$.
\end{proof}
Let us make one simple observation about the sequence $Y_1, \ldots, Y_k$ and $Y_1', \ldots, Y_k'$.
Since only one datapoint has been changed, it must be the case that $|\{i \in [k] \,:\, Y_i \neq Y_i'\}| \leq 1$.
Without loss of generality (perhaps by reindexing), we assume that $Y_i = Y_i'$ for $2 \leq i \leq k$.
In addition, we have the following straightforward claim which follows from the aforementioned observation.
\begin{claim}
    \ClaimName{W_bound}
    For $2 \leq i \leq k$, we have $|w_i - w_i'| \leq 10/k$.
    In addition, $|W - W'| \leq 11$.
\end{claim}
\begin{proof}
    That $|w_i - w_i'| \leq 10/k$ for $2 \leq i \leq k$ follows from the fact that $Y_i = Y_i'$ for $2 \leq i \leq k$
    (and because \Line{WeightDef} implies that changing a single datapoint can cause $w_i$ can change by at most $10/k$).
    Next, $|W - W'| \leq |w_1 - w_1'| + \sum_{2 \leq i \leq k} |w_i - w_i'| \leq 1 + (k-1) \cdot 10 / k < 11$.
\end{proof}

Now, let us proceed to bound $\dist(\mu, \mu')$.
Using repeated applications of convexity, we have
\begin{align}
    \dist(\mu, \mu')
    & = \dist\left( \frac{w_1}{W} Y_1 + \sum_{i=2}^k \frac{w_i}{W} Y_i, \frac{w_1'}{W'} Y_1' + \sum_{j=2}^k \frac{w_j'}{W'} Y_j \right) \notag \\
    & \leq \frac{w_1}{W} \dist\left( Y_1, \frac{w_1'}{W'} Y_1' + \sum_{j=2}^k \frac{w_j'}{W'} Y_j \right) \notag \\
    & +
    \left(1 - \frac{w_1}{W} \right) \dist\left( \frac{1}{1-w_1/W} \sum_{i=2}^k \frac{w_i}{W} Y_i, \frac{w_1'}{W'} Y_1' + \sum_{j=2}^k \frac{w_j'}{W'} Y_j \right) \notag \\
    & \leq \frac{w_1}{W}\frac{w_1'}{W'} \dist\left( Y_1, Y_1'\right) \EquationName{dist1} \\
    & + \frac{w_1}{W} \left( 1 - \frac{w_1'}{W'} \right) \dist\left( Y_1, \frac{1}{1-w_1'/W'} \sum_{j=2}^k \frac{w_j'}{W'} Y_j \right) \EquationName{dist2} \\
    & + \left(1 - \frac{w_1}{W}\right) \frac{w_1'}{W'} \dist\left( \frac{1}{1-w_1/W} \sum_{i=2}^k \frac{w_i}{W} Y_i, Y_1' \right) \EquationName{dist3} \\
    & + \left(1 - \frac{w_1}{W} \right) \left(1 - \frac{w_1'}{W'} \right) \dist\left( 
        \frac{1}{1-w_1/W} \sum_{i=2}^k \frac{w_i}{W} Y_i, \frac{1}{1-w_1'/W'} \sum_{j=2}^k \frac{w_i'}{W'} Y_j \EquationName{dist4}
    \right),
\end{align}
where we used convexity in the first argument for the first inequality and convexity in the second argument for the second inequality.
We now bound each \Equation{dist1}, \Equation{dist2}, \Equation{dist3}, and \Equation{dist4} separately.
To do so, we make use of the following claim.
\begin{claim}
    \ClaimName{PrivacyTriangle}
    Suppose that $w_i, w_j' > 0$. Then $\dist(Y_i, Y_j') \leq 2r$.
\end{claim}
\begin{proof}
    Let $S_i = \{ s \in [k] \,:\, \dist(Y_i, Y_s) \leq r / t\}$ and $S_j' = \{s \in [k] \,:\, \dist(Y_j', Y_s') \leq r / t\}$.
    Since $w_i, w_j'>0$, we must have $|S_i| \geq 0.6k$ and $|S_j'| \geq 0.6k$.
    Hence, $|S_i \cap S_j'| = |S_i| + |S_j'| - |S_i \cup S_j'| \geq 0.2k \geq 2$ since $k \geq 10$.
    So there exists $s \in S_i \cap S_j'$ such that $2 \leq s \leq k$ which implies that $Y_s = Y_s'$.
    In particular, for this $s$, we have $\dist(Y_i, Y_s) \leq r/t$ and $\dist(Y_j', Y_s) \leq r/t$.
    Recalling that $\dist$ satisfies a $t$-approximate $r$-restricted triangle inequality, we have
    $\dist(Y_i, Y_j') \leq t(\dist(Y_i, Y_s) + \dist(Y_j', Y_s)) \leq 2r$.
\end{proof}
By \Claim{PrivacyTriangle}, we have that \Equation{dist1} is bounded by 
\begin{equation}
    \EquationName{distBound1}
    \frac{w_1 w_1'}{WW'} \dist(Y_1, Y_1') \leq \frac{w_1 w_1'}{WW'} 2r
\end{equation}
We now bound \Equation{dist2} which is very similar to the bound of \Equation{dist1} we just obtained.
Using convexity, we have
\begin{equation}
    \EquationName{distBound2}
    \begin{aligned}
    \frac{w_1}{W} \left( 1 - \frac{w_1'}{W'} \right) \dist\left( Y_1, \frac{1}{1-w_1'/W'} \sum_{j=2}^k \frac{w_j'}{W'} Y_j \right)
    & \leq \frac{w_1}{W} \sum_{j=2}^k \frac{w_j'}{W'} \dist(Y_1, Y_j) \\
    & \leq \frac{w_1}{W} \sum_{j=2}^k \frac{w_j'}{W'} \cdot 2r,
    \end{aligned}
\end{equation}
where we applied \Claim{PrivacyTriangle} for the last inequality.
Similarly, we have
\begin{equation}
    \EquationName{distBound3}
    \begin{aligned}
    \left(1 - \frac{w_1}{W}\right) \frac{w_1'}{W'} \dist\left( \frac{1}{1-w_1/W} \sum_{i=2}^k \frac{w_i}{W} Y_i, Y_1' \right)
    \leq \frac{w_1'}{W'} \sum_{i=2}^k \frac{w_j}{W} 2r.
    \end{aligned}
\end{equation}
Note that \Equation{distBound1}, \Equation{distBound2}, \Equation{distBound3} shows that the sum of \Equation{dist1}, \Equation{dist2}, \Equation{dist3} is bounded above by $\frac{2r}{W} + \frac{2r}{W'}$.

Finally, we bound \Equation{dist4}.
We require the following claim whose proof is similar to that of \Claim{PrivacyTriangle}.
\begin{claim}
    \ClaimName{PrivacyTriangle2}
    Let $S = \{ 2 \leq i \leq k \,:\, \max(w_i, w_i') > 0 \}$.
    Then for all $i, j \in S$, we have $\dist(Y_i, Y_j) \leq 2r$.
\end{claim}
\begin{proof}
    Let $S_i = \{ s \in [k] \,:\, \dist(Y_i, Y_s) \leq r / t\}$ and $S_j = \{s \in [k] \,:\, \dist(Y_j, Y_s) \leq r / t\}$.
    Since $\max(w_i, w_i'), \max(w_j, w_j') > 0$, we must have $|S_i| \geq 0.6k-1$ and $|S_j| \geq 0.6k-1$.
    Hence, $|S_i \cap S_j| = |S_i| + |S_j| - |S_i \cup S_j| \geq 0.2k - 2 > 1$ so $|S_i \cap S_j| \geq 2$.
    So there exists $s \in S_i \cap S_j$ such that $s \geq 2$.
    In particular, for this $s$, we have $\dist(Y_i, Y_s) \leq r/t$ and $\dist(Y_j, Y_s) \leq r/t$.
    Recalling that $\dist$ satisfies a $t$-approximate $r$-restricted triangle inequality, we have
    $\dist(Y_i, Y_j) \leq t(\dist(Y_i, Y_s) + \dist(Y_j, Y_s)) \leq 2r$.
\end{proof}

Note that $\sum_{i=2}^k \frac{w_i}{W} Y_i = \sum_{i \in S} \frac{w_i}{W} Y_i$
and $\sum_{i=2}^k \frac{w_i'}{W'} Y_i = \sum_{i \in S} \frac{w_i'}{W'} Y_i$.
Now, we have
\begin{align*}
    \dist & \left( \frac{1}{1-w_1/W} \sum_{i\in S} \frac{w_i}{W}Y_i, \frac{1}{1-w_1'/W'} \sum_{i\in S} \frac{w_i'}{W'} Y_i \right) \\
    & \leq\left({\phi} + \max_{i, j \in S} \dist(Y_i, Y_j) \right) \sum_{i\in S}\left| \frac{w_i/W}{1-w_1/W} - \frac{w_i'/W'}{1-w_1'/W'} \right|.
\end{align*}

Where we used the fact that $\dist$ satisfies locality in \Definition{semimetric}. Note that $\dist(Y_i, Y_j) \leq 2r$ for all $i, j \in S$.
Also, by the \Claim{Delta_Weights} we see that the right hand side is bounded by $100/k$. Therefore, the term in \Equation{dist4} is upper bounded by $100(2r+\phi)/k$. We now state and prove \Claim{Delta_Weights}.

\begin{claim}
\ClaimName{Delta_Weights}
    $\sum_{i\in S}\left| \frac{w_i/W}{1-w_1/W} - \frac{w_i'/W'}{1-w_1'/W'} \right| \leq 100/k$.
\end{claim}
\begin{proof}
\begin{align*}
    &\sum_{i\in S}\left| \frac{w_i/W}{1-w_1/W} - \frac{w_i'/W'}{1-w_1'/W'} \right|=\sum_{i\in S}\left| \frac{w_i}{W-w_1} - \frac{w_i'}{W'-w_1'} \right|\\
    &\leq \sum_{i\in S}\left( \left| \frac{w_i}{W-w_1} - \frac{w_i}{W'-w_1'} \right| + \left| \frac{w_i}{W'-w_1'} - \frac{w_i'}{W'-w_1'} \right|\right)\\
    &\leq \sum_{i\in S}\left( \frac{w_i}{W}\left| \frac{(W'-W)+(w_1-w_1')}{(1-w_1/W)(W'-w_1')}\right| +  |w_i-w_i'| \left|\frac{1}{W'-w_1'}\right| \right)\\
    &\leq \sum_{i\in S}\left( \frac{w_i}{W}\left| \frac{11+2}{(1 - 3/k)(k/3 - 12)}\right| +  |10/k|\frac{1}{k/3 - 12}  \right)\\
    &\leq \left(\frac{39}{0.9(k-36)}\right) + \frac{30}{k-36}
    \leq \frac{74}{k - 36} \leq 100/k.
\end{align*}
In the second last line, we used that $|1-w_1/W| = 1 - w_1/W \geq 1 - 1/W \geq 1 - 3/k$ by \Claim{W_LB} (provided $k > 3$) and $|W' - w_1'| = W' - w_1' \geq W - |W - W'| - w_1' \geq k/3 - 12$ by \Claim{W_LB} and \Claim{W_bound} (provided $k > 36$).
For the last line, we used that $\sum_{i \in S} w_i / W \leq 1$, $\sum{i \in S} 1/k \leq 1$, and $1-3/k \geq 0.9$ for $k \geq 30$.
Finally, the last inequality is true provided $k \geq 140$.

We have now bounded \Equation{distBound1}, \Equation{distBound2}, \Equation{distBound3}, and \Equation{dist4}. The sum of these terms is at most $\frac{2r}{W} + \frac{2r}{W'} + 100(2r+\phi)/k$. Using the \Claim{W_LB} again we can upper bound this sum by 
$(212r+100\phi)/k < 400(r+\phi)/k$.
\end{proof}

We now verify that the algorithm is $(2\eps, {4e^\eps\delta})$-DP in the present case.
Since $\dist(\mu, \mu') \leq 400(r + \phi)/k$ and $\cB$ is a $(400(r+\phi)/k, \eps, \delta)$-masking mechanism,
we know that $\cB(\mu), \cB(\mu')$ are $(\eps, \delta)$-indistinguishable.
Let $T \subseteq \cY$.
Then
\begin{align*}
    \prob{\cM(D) \in T}
    & = \prob{\cM(D) \neq \perp} \prob{\cB(\mu) \in T} \\
    & \leq (e^{\eps} \prob{\cM(D') \neq \perp} + \delta) (e^{\eps} \prob{\cB(\mu') \in T} + \delta) \\
    & = e^{2\eps} \prob{\cM(D') \neq \perp} \prob{\cB(\mu') \in T} \\
    & + e^{\eps} \delta (\prob{\cM(D') \neq \perp} + \prob{\cB(\mu') \in T})
      + \delta^2 \\
    & \leq e^{2\eps} \prob{\cM(D') \in T} + 3e^{\eps} \delta,
\end{align*}
which establishes \Equation{ReductionDP1}.
Finally,
\begin{align*}
    \prob{\cM(D) \in T \cup \{\perp\}}
    & = \prob{\cM(D) = \perp} + \prob{\cM(D) \in T} \\
    & \leq (e^{\eps} \prob{\cM(D') = \perp} + \delta) + (e^{2\eps} \prob{\cM(D') \in T} + 3e^{\eps} \delta) \\
    & \leq e^{2\eps} \prob{\cM(D') \in T \cup \{\perp\}} + 4e^{\eps} \delta,
\end{align*}
which establishes \Equation{ReductionDP2}.
\end{proof}

%% file: app_covariance.tex
\section{Missing proofs from \Section{covariance}}
\AppendixName{covariance}
\subsection{Proof of \Lemma{CovarianceTechnical}}
\AppendixName{CovarianceTechnical}
\Lemma{CovarianceTechnical} follows easily from the following lemma.
\begin{lemma}
    \LemmaName{covarIndistinguishable}
    Suppose that
    \[
        \max\{\| \Sigma_1^{-1/2} \Sigma_2 \Sigma_1^{-1/2} - I_d \|, \|\Sigma_2^{-1/2} \Sigma_1 \Sigma_2^{-1/2} - I_d\|\} \leq \gamma \leq 1/2.
    \]
    Then for every $\delta_0 > 0$, $\Sigma_1^{1/2}(I + \eta G)$ and $\Sigma_2^{1/2}(I + \eta G)$ are $(\eps_0(\delta_0), \delta_0)$-indistinguishable where
    \begin{equation}
        \EquationName{covarIndistinguishable}
        \eps_0(\delta_0) = 
        \frac{\gamma^2d}{2}(d + 1/\eta^2) + 2 d \gamma \sqrt{\ln(2/\delta_0)} + 2 \gamma \ln(2/\delta_0) + 3\gamma \sqrt{d} \sqrt{\ln(2/\delta_0)}/\eta.
    \end{equation}
\end{lemma}
Indeed, given \Lemma{covarIndistinguishable}, \Lemma{CovarianceTechnical} follows by setting $\gamma$ appropriately
so that each term in \Equation{covarIndistinguishable} is at most $\eps / 4$ and post-processing,
i.e.~since $\Sigma_1^{1/2}(I + \eta G)$ and $\Sigma_2^{1/2}(I + \eta G)$ are $(\eps, \delta)$-indistinguishable
then so are $\Sigma_1^{1/2}(I + \eta G)(I + \eta G)\transpose \Sigma_1^{1/2}$ and $\Sigma_2^{1/2}(I + \eta G)(I + \eta G)\transpose \Sigma_2^{1/2}$.

First, we require the following technical lemma which asserts that for Gaussians $\cN(\mu, \Sigma), \cN(\wmu, \wSigma)$,
if $g \sim \cN(0, I_d)$ then the random variable $\privLoss{\cN(\mu, \Sigma)}{\cN(\widehat{\mu}, \widehat{\Sigma})}(\Sigma^{1/2} g + \mu)$
concentrates tightly around $\KL{\cN(\mu, \Sigma)}{\cN(\what{\mu}, \what{\Sigma})}$.
The plan is to apply the lemma with $\mu \in \bR^{d^2}$ (resp.~$\what{\mu}$) is $\Sigma_1^{1/2}$ (resp.~$\Sigma_2^{1/2}$) but viewed as a vector of length $d^2$
and $\Sigma$ (resp.~$\what{\Sigma}$) is the covariance matrix of (the vector) $\Sigma_1^{1/2}(I + \eta G)$ (resp.~$\Sigma_2^{1/2}(I + \eta G)$).

\begin{lemma}
	\LemmaName{privLossNormal}
    Let $\mu, \wmu \in \bR^d$ and $\Sigma, \wSigma \in \pd^d$.
    Let $g \sim \cN(0, I_d)$.
    Then
	\begin{align*}
        \privLoss{\cN(\mu, \Sigma)}{\cN(\wmu, \wSigma)}(\Sigma^{1/2} g + \mu)
		& \leq \KL{\cN(\mu, \Sigma)}{\cN(\wmu, \wSigma)} \\
		& + 2\|\Sigma^{1/2}\wSigma^{-1} \Sigma^{1/2}-I\|_F \sqrt{\ln(2/\delta)} + 2\|\Sigma^{1/2}\wSigma^{-1} \Sigma^{1/2}-I\| \ln(2/\delta) \\
		& + \|\Sigma^{1/2} \wSigma^{-1} \Sigma^{1/2}\| \cdot \|\Sigma^{-1/2}(\wmu - \mu)\|_2 \sqrt{2\ln(2/\delta)}
	\end{align*}
	with probability at least $1-\delta$.
\end{lemma}
The proof of \Lemma{privLossNormal} can be found in \Subsection{privLossNormal}.

We use \Lemma{privLossNormal} to show that if $\| \Sigma_1^{-1/2} \Sigma_2 \Sigma_1^{-1/2} - I \|$ and $\| \Sigma_2^{-1/2} \Sigma_1 \Sigma_2^{-1/2} - I \|$
are small then $\Sigma_1^{1/2}(I_d + \eta G)$ and $\Sigma_2^{1/2}(I_d + \eta G)$
are indistinguishable.
As a consequence, the matrices $\Sigma_1^{1/2}(I_d + \eta G)(I_d + \eta G)\transpose \Sigma_1^{1/2}$ and $\Sigma_2^{1/2}(I_d + \eta G)(I_d + \eta G)\transpose \Sigma_2^{1/2}$
are also indistinguishable.

Before we prove \Lemma{covarIndistinguishable}, it is convenient to set up some notation.
For a matrix $B = \begin{bmatrix} b_1 \cdots b_d \end{bmatrix} \in \bR^{d \times d}$, we write $\vvec(B)$
to denote a vector of length $d^2$ with the columns of $B$ stacked together.
More precisely,
\[
    \vvec(B) = \begin{bmatrix} b_1 \\ \vdots \\ b_d \end{bmatrix} \in \bR^{d^2}.
\]
Next, for a matrix $A = \begin{bmatrix} a_1 \cdots a_d \end{bmatrix} \in \bR^{k \times d}$, we let $\bd(A)$
be the block-diagonal matrix with $d$ copies of $A$ along its diagonal, i.e.
\[
	\bd(A) = \begin{bmatrix} A & 0 & \ldots & 0 \\ 0 & A & \ldots & 0 \\ \vdots & \vdots & \ddots & \vdots \\  0 & 0 & \ldots & A \end{bmatrix} \in \bR^{kd \times d^2}.
\]
Note that for a size-compatible matrix $\vvec(AB) = \bd(A) \vvec(B)$.
Indeed, $AB = \begin{bmatrix} Ab_1 \cdots Ab_d \end{bmatrix}$ so
\[
    \underbrace{\begin{bmatrix} Ab_1 \\ \vdots \\ Ab_d \end{bmatrix}}_{\vvec(AB)}
    = \underbrace{\begin{bmatrix} A & 0 & \ldots & 0 \\ 0 & A & \ldots & 0 \\ \vdots & \vdots & \ddots & \vdots \\  0 & 0 & \ldots & A \end{bmatrix}}_{\bd(A)}
    \underbrace{\begin{bmatrix} b_1 \\ \vdots \\ b_d \end{bmatrix}}_{\vvec(B)}.
\]
We require a few calculations before we can prove \Lemma{covarIndistinguishable}.

\begin{claim}
    \ClaimName{KLBound1}
    Suppose that $\|\Sigma_1^{1/2} \Sigma_2^{-1} \Sigma_1^{1/2} - I_d \| \leq \gamma$.
    Then
    \[
        \|\bd(\Sigma_1)^{1/2} \bd(\Sigma_2)^{-1} \bd(\Sigma_1)^{1/2} - I_{d^2}\|_F^2 = d \cdot \|\Sigma_1^{-1/2} \Sigma_2 \Sigma_1^{-1/2} - I_d \|_F^2 \leq d^2 \gamma^2.
    \]
\end{claim}
\begin{proof}
    The equality is because the matrices $\bd(\Sigma_1), \bd(\Sigma_2), I_{d^2}$ are all block diagonal
    with copies of $\Sigma_1, \Sigma_2, I_d$ along its diagonal.

    For the inequality, we have $\| \Sigma_1^{1/2} \Sigma_2^{-1} \Sigma_1^{1/2} - I_d \| \leq \gamma$ which implies
    $\|\Sigma_1^{1/2} \Sigma_2^{-1} \Sigma_1^{1/2} - I_d \|_F^2 \leq \gamma^2 \cdot d$.
    Multiplying both sides by $d$ gives the result.
\end{proof}

\begin{claim}
    \ClaimName{KLBound2}
    Suppose that $\|\Sigma_1^{1/2} \Sigma_2^{-1} \Sigma_1^{1/2} - I_d \| \leq \gamma$.
    Then $\|\eta^{-1} \bd(\Sigma_2)^{-1/2} (\vvec(\Sigma_2^{1/2}) - \vvec(\Sigma_1^{1/2})) \|_2^2 \leq \gamma^2 d / \eta^2$.
\end{claim}
\begin{proof}
    Note that 
    $\|\bd(\Sigma_2)^{-1/2} (\vvec(\Sigma_2^{1/2}) - \vvec(\Sigma_1^{1/2})) \|_2^2 = \|I - \Sigma_2^{-1/2} \Sigma_1^{1/2}\|_F^2$.
    Let $\lambda_1, \ldots, \lambda_d$ be the (positive) eigenvalues of $\Sigma_1^{1/2} \Sigma_2^{-1} \Sigma_1^{1/2}$.\footnote{
        The eigenvalues of $\Sigma_2^{-1/2} \Sigma_1 \Sigma_2^{-1/2} - I_d$ lie in the interval $[-1/2, 1/2]$
        so the eigenvalues of $\Sigma_2^{-1/2} \Sigma_1 \Sigma_2^{-1/2}$ lie in the interval $[1/2, 3/2]$.
        In particular, the eigenvalues are positive.}
    Then $\sqrt{\lambda_1}, \ldots, \sqrt{\lambda_d}$ are the singular values of $\Sigma_2^{-1/2} \Sigma_1^{1/2}$
    and $1 - \sqrt{\lambda_1}, \ldots, 1-\sqrt{\lambda_d}$ are the singular values of $I_d - \Sigma_2^{-1/2} \Sigma_1^{1/2}$.
    Thus,
    \begin{align*}
        \|I_d - \Sigma_2^{-1/2} \Sigma_1^{1/2}\|_F^2
        & = \sum_{i=1}^d (1 - \sqrt{\lambda_i})^2 \\
        & = \sum_{i=1}^d 1 - 2\sqrt{\lambda_i} + \lambda_i \\
        & \leq \sum_{i=1}^d 1 - 2\sqrt{1-\gamma} + (1+\gamma) \\
        & \leq \sum_{i=1}^d 1 - 2(1 - \gamma / 2 - \gamma^2 / 2) + (1 + \gamma) \\
        & = \sum_{i=1}^d \gamma^2 = \gamma^2 d,
    \end{align*}
    where in the third line we used that $\lambda_i \in [1-\gamma, 1+\gamma]$ for all $i \in [d]$
    and in the fourth line, we used the assumption that $\gamma \leq 1/2$ and
    \Claim{sqrtIneq} to bound $\sqrt{1-\gamma} \leq 1 - \gamma / 2 - \gamma^2 / 2$.
    The lemma follows by dividing both sides by $\eta^2$.
\end{proof}
\begin{claim}
    \ClaimName{KLOutput}
    Suppose that $\|\Sigma_1^{1/2} \Sigma_2^{-1} \Sigma_1^{1/2} - I_d \| \leq \gamma \leq 1/2$.
    Then
    \[
        \KL{\cN(\vvec(\Sigma_1^{1/2}), \eta^2 \bd(\Sigma_1))}{\cN(\vvec(\Sigma_2^{1/2}), \eta^2 \bd(\Sigma_2))} \leq \frac{\gamma^2 d}{2} (d + 1/\eta^2).
    \]
\end{claim}
\begin{proof}
    Follows by plugging in the bounds from \Claim{KLBound1} and \Claim{KLBound2} into \Fact{KLNormal}.
\end{proof}

\begin{proof}[Proof of \Lemma{covarIndistinguishable}]
    It is convenient to view $\Sigma_i^{1/2}(I_d + \eta G)$ as $d^2$-dimensional vectors.
    Indeed, we have $\vvec(\Sigma_i^{1/2}(I_d + \eta G)) = \vvec(\Sigma_i^{1/2}) + \eta \bd(\Sigma_i^{1/2}) \vvec(G)$
    which is a random $d^2$-dimensional vector distributed as $\cN(\vvec(\Sigma_i^{1/2}), \eta^2 \bd(\Sigma_i))$.
    For notation, let us write $\cD_i = \cN(\vvec(\Sigma_i^{1/2}), \eta^2 \bd(\Sigma_i))$.
    We will use \Lemma{privLossNormal} to bound 
    \[
        \cL \coloneqq \privLoss{\cD_1}{\cD_2}(\eta \bd(\Sigma_1)^{1/2} \vvec(G) + \vvec(\Sigma_1)^{1/2}).
    \]
    An identical argument holds with $\Sigma_1, \Sigma_2$ reversed and so we may apply \Lemma{privLossDP}.

    By \Lemma{privLossNormal}, we have that
    \begin{equation}
        \EquationName{privLossOutput}
        \begin{aligned}
            \cL
            & \leq \KL{\cD_1}{\cD_2} \\
            & + 2 \| \bd(\Sigma_1)^{1/2} \bd(\Sigma_2)^{-1} \bd(\Sigma_1)^{1/2} - I \|_F \sqrt{\ln(2/\delta)} \\
            & + 2 \|\bd(\Sigma_1)^{1/2} \bd(\Sigma_2)^{-1} \bd(\Sigma_1)^{1/2} - I_{d^2} \| \ln(2/\delta) \\
            & + \|\bd(\Sigma_1)^{1/2} \bd(\Sigma_2)^{-1} \bd(\Sigma_1)^{1/2} \| \cdot \|\eta^{-1}\bd(\Sigma_1)^{-1/2} (\vvec(\Sigma_2) - \vvec(\Sigma_1)) \|_2 \sqrt{2\ln(2/\delta)},
        \end{aligned}
    \end{equation}
    with probability at least $1-\delta$.
    We now bound each term in \Equation{privLossOutput}.

    The first term in \Equation{privLossOutput} is bounded by \Claim{KLOutput} and we have
    \begin{equation}
        \EquationName{privLossOutput1}
        \KL{\cD_1}{\cD_2} \leq \frac{\gamma^2d}{2} (d + 1/\eta^2).
    \end{equation}
    The second term is bounded by \Claim{KLBound1} and we have
    \begin{equation}
        \EquationName{privLossOutput2}
        \|\bd(\Sigma_1)^{1/2} \bd(\Sigma_2)^{-1} \bd(\Sigma_1)^{1/2} - I_{d^2} \|_F \leq \gamma d.
    \end{equation}
    For the third term, note that $\bd(\Sigma_1)^{1/2} \bd(\Sigma_2)^{-1} \bd(\Sigma_1)^{1/2}$ and $\Sigma_1^{1/2} \Sigma_2^{-1} \Sigma_1$
    have the same set of eigenvalues but with different multiplicities (by a factor of $d$).
    Hence,
    \begin{equation}
        \EquationName{privLossOutput3}
        \|\bd(\Sigma_1)^{1/2} \bd(\Sigma_2)^{-1} \bd(\Sigma_1)^{1/2} - I_{d^2} \| \leq \gamma.
    \end{equation}
    Finally, for the last term, we use that $\|\bd(\Sigma_1)^{1/2} \bd(\Sigma_2)^{-1} \bd(\Sigma_1)^{1/2} \| \leq 1 + \gamma \leq 3/2$ and \Claim{KLBound2} to obtain
    \begin{equation}
        \EquationName{privLossOutput4}
        \|\bd(\Sigma_1)^{1/2} \bd(\Sigma_2)^{-1} \bd(\Sigma_1)^{1/2} \| \cdot \|\eta^{-1}\bd(\Sigma_1)^{-1/2} (\vvec(\Sigma_2) - \vvec(\Sigma_1)) \|_2 < 3 \gamma \sqrt{d} / 2 \eta.
    \end{equation}
    Plugging \Equation{privLossOutput1} through \Equation{privLossOutput4} into \Equation{privLossOutput} gives
    \[
        \cL \leq \frac{\gamma^2 d}{2} (d + 1/\eta^2) + 2 d \gamma \sqrt{\ln(2/\delta)} + 2\gamma \ln(2/\delta) + 3\gamma \sqrt{d} \sqrt{\ln(2/\delta)} / \eta,
    \]
    with probability at least $1-\delta$.
\end{proof}

\subsection{Proof of \Lemma{privLossNormal}}
\SubsectionName{privLossNormal}
\begin{proof}[Proof of \Lemma{privLossNormal}]
    Let $\cL = \privLoss{\cN(\mu, \Sigma)}{\cN(\wmu, \wSigma)}(\Sigma^{1/2} g + \mu)$.
    From \Claim{privLossNormal}, we have that
    \begin{equation}
        \EquationName{privLoss1}
		\cL
		= -\frac{1}{2} \ln(\det(\wSigma^{-1} \Sigma)) - \frac{1}{2} g\transpose g
		  + \frac{1}{2} (g - \Sigma^{-1/2} (\wmu - \mu))\transpose \Sigma^{1/2} \wSigma^{-1} \Sigma^{1/2} (g - \Sigma^{-1/2} (\wmu - \mu)).
    \end{equation}
    Expanding the last term, we have
    \begin{equation}
        \EquationName{privLoss2}
        \cL = -\frac{1}{2} \ln(\det(\wSigma^{-1} \Sigma)) + \frac{1}{2} g\transpose (\Sigma^{1/2} \wSigma^{-1} \Sigma^{1/2} - I)g - (\wmu - \mu)\transpose \wSigma^{-1} \Sigma^{1/2} g
        + \frac{1}{2} (\wmu - \mu)\transpose \wSigma^{-1} (\wmu - \mu).
    \end{equation}
    By \Lemma{chiSquareConcentration}, we have that
    \begin{equation}
        \EquationName{privLoss3}
        \begin{aligned}
            g \transpose (\Sigma^{1/2} \wSigma^{-1} \Sigma^{1/2} - I) g
            & \leq \Tr(\Sigma^{1/2} \wSigma^{-1} \Sigma^{1/2} - I) + 4\|\Sigma^{1/2} \wSigma^{-1} \Sigma^{1/2} - I\|_F \sqrt{\ln(2/\delta)} \\ 
            & + 4 \|\Sigma^{1/2} \wSigma^{-1} \Sigma^{1/2} - I\| \ln(2/\delta),
        \end{aligned}
    \end{equation}
    with probability $1-\delta/2$.
    Next, $(\wmu - \mu)\transpose \wSigma^{-1} \Sigma^{1/2} g$ is a mean-zero Gaussian random variable with variance
    $\|\Sigma^{1/2} \wSigma^{-1} (\wmu - \mu)\|_2^2 \leq \|\Sigma^{1/2} \wSigma^{-1} \Sigma^{1/2}\|^2 \|\Sigma^{-1/2}(\wmu - \mu)\|_2^2$.
    Hence, using the tails of a Gaussian random variable (\Lemma{gaussConcentration}), we have
    \begin{equation}
        \EquationName{privLoss4}
        (\wmu - \mu)\transpose \wSigma^{-1} \Sigma^{1/2} g \leq \|\Sigma^{1/2} \wSigma^{-1} \Sigma^{1/2} \| \|\Sigma^{-1/2} (\wmu - \mu)\|_2 \sqrt{2 \ln(2/\delta)}
    \end{equation}
    with probability $1-\delta/2$.
    Plugging \Equation{privLoss3} and \Equation{privLoss4} along with the formula for the KL divergence between two multivariate Normal distributions (\Fact{KLNormal}), we have
    \begin{equation}
        \EquationName{privLoss5}
        \begin{aligned}
            \cL
            & \leq \KL{\cN(\mu, \Sigma)}{\cN(\wmu, \wSigma)} + 2\|\Sigma^{1/2} \wSigma^{-1} \Sigma^{1/2} - I\|_F \sqrt{\ln(2/\delta)} \\
            & + 2 \|\Sigma^{1/2} \wSigma^{-1} \Sigma^{1/2}  - I\| \ln(2/\delta)
            + \|\Sigma^{1/2} \wSigma^{-1} \Sigma^{1/2} \| \|\Sigma^{-1/2}(\wmu - \mu)\|_2 \sqrt{2\ln(2/\delta)},
        \end{aligned}
    \end{equation}
    with probability $1-\delta$.
\end{proof}